\documentclass{article}
\usepackage{iclr2027_conference,times}

\usepackage[T1]{fontenc}
\usepackage{graphicx,colordvi,psfrag}
\usepackage{amsmath,amssymb,algpseudocode}
\usepackage{algorithm}
\usepackage{algorithmicx}
\usepackage{calc,pstricks, pgf, xcolor}
\usepackage{bbding}
\usepackage{bbm}
\usepackage{dsfont}
\usepackage{stfloats}
\usepackage{hyperref}
\usepackage{xparse}
\usepackage{enumerate}
\usepackage{tikz}
\usetikzlibrary{shapes,arrows,positioning,calc}

\usepackage{subcaption}

\newcommand{\ZZ}{\mathbb{Z}}
\newcommand{\RR}{\mathbb{R}}

\newcommand{\diag}{\mathop{\mathrm{diag}}}

\newcommand{\trace}{\mathop{\mathrm{tr}}}
\def\tr{\trace}

\newcommand{\Unif}{\mathrm{Uniform}}
\newcommand\indep{\protect\mathpalette{\protect\independenT}{\perp}}
\def\independenT#1#2{\mathrel{\rlap{$#1#2$}\mkern2mu{#1#2}}}

\def\EE{\mathbb{E}}
\newcommand{\m}{\mathcal}

\newcommand{\FP}{\mathrm{FP}}

\newcommand{\Qt}{\tilde{Q}_{\FP 4}}
\newcommand{\Cscale}{\mathcal{C}_{\mathrm{scale}}}
\newcommand{\Cfp}{\mathcal{C}_{\mathrm{FP4}}}

\NewDocumentCommand{\DIP}{e{^_}}{D^{\mathrm{IP}\IfValueT{#1}{,#1}}_{\IfValueT{#2}{#2}}}

\NewDocumentCommand{\DMM}{e{^_}}{D^{\mathrm{MM}\IfValueT{#1}{,#1}}_{\IfValueT{#2}{#2}}}

\DeclareMathOperator*{\argmin}{\arg\!\min}

\newtheorem{proposition}{Proposition}

\newtheorem{lemma}{Lemma}

\newenvironment{proof}[1][Proof]{\noindent\textbf{#1.} }{\ \rule{0.5em}{0.5em}}

\usepackage{hyperref}
\usepackage{url,booktabs}
\usepackage{amsmath}

\title{OSFP4: Joint Optimization of Diagonal Smoothing and Block Scales for NVFP4 Quantization}

\author{Neriah Ben David$^{1}$, Ori Meir$^{1}$ \& Or Ordentlich$^{1,2}$\thanks{The authors are listed in alphabetical order. } \\
$^{1}$School of Computer Science and Engineering\\
Hebrew University of Jerusalem\\
Jerusalem, Israel\\
$^{2}$Pearl Research Labs\\
\texttt{\{neriah.bendavid,ori.meir,or.ordentlich\}@mail.huji.ac.il}
}

\iclrfinalcopy

\begin{document}

\maketitle
\fancyhead{}
\renewcommand{\headrulewidth}{0pt}

\begin{abstract}
NVFP4 is an attractive datatype for large language model (LLM) inference, offering compact storage and native tensor-core acceleration. However, preserving accuracy using NVFP4 requires careful quantization.
In this work we develop a novel quantization scheme called Optimized Smoothing and Scaling for NVFP4 (OSFP4). For each linear projection it uses a diagonal smoothing matrix whose entries are optimized to minimize the squared matrix-product quantization error under NVFP4, taking into account the rounding procedure that is used (either round-to-nearest, or GPTQ-style successive interference cancellation). This requires performing joint optimization on the smoothing entries as well as the block scales, which is facilitated by analyzing a multiplicative-dither FP4 quantizer instead of the fixed deterministic one. Experiments show that OSFP4 achieves the highest average accuracy among the evaluated competitors in the corresponding quantization settings, while retaining approximately 94-97\% of vendor NVFP4 prefill throughput on the measured workloads. Our code is available \href{https://github.com/neriahbd/OSFP4}{here}.
\end{abstract}

\section{Introduction}

Low-bit data formats are gaining much popularity for large language models (LLMs) inference, as they require less memory for storing and running the model, decrease the traffic from and to memory, and allow to execute matrix multiplication (MatMul) faster through optimized hardware. The FP4 datatype stands out as a prominent representative of such formats which is heavily used in models running on NVIDIA's Blackwell series and its subsequent generations. Unfortunately, using FP4 instead of higher resolution formats such as FP8 often comes with the cost of reduced accuracy. The goal of this paper is to develop a quantization scheme for NVFP4 that jointly optimizes a \emph{diagonal smoothing transformation, as well as the weight block scales} in order to maintain maximal accuracy. To that end, we first need to derive an approximate expression for the squared error of the NVFP4 quantized MatMul, as a function of the diagonal smoothing parameters and the weight block scales.

For floating-point (FP) constellations with $E\gg 1$ exponent bits and $M$ mantissa bits, the infinite-exponent independent noise (IEIN) model of~\citet{OP26surveyI} characterizes the quantizer with input $y\in\RR$ as an additive noise channel $\tilde{Q}_{\FP}(y)=y+\sqrt{|y|^2\cdot c_{\FP}M^{-2}}\cdot Z$, where $Z$ is statistically independent of $y$ and has zero mean and unit variance, and $c_{\FP}>0$ is an explicitly defined constant. For FP8, it is shown in~\cite{OP26surveyI} that this simplified model typically provides an excellent prediction for the squared error of FP quantized MatMul.
Unfortunately, the FP4 E2M1 format only has $E=2$ exponent bits and the dynamic range of the corresponding constellation is therefore highly limited. Consequently, the IEIN model fails to capture the true behavior of the FP4 quantization error. 

To address this challenge, we extend the IEIN model to one that  captures finite dynamic range effects. In particular, we model the FP4 quantizer as 
\begin{align}
	\Qt(y)=y+\sqrt{|y|^2\phi(y)}Z,
	\label{eq:qterrormodel}
\end{align}
where $Z$ is statistically independent of $y$ and has zero mean and unit variance, and
\begin{align}
	\phi(y)=\frac{1}{y^2}\EE\left(y-2^U Q_{\mathrm{FP4}}(2^{-U}y)\right)^2 ,~~U\sim\Unif([0,1)),~ U\indep y.  
    \label{eq:PhiDef}
\end{align}
The mean squared error for the model~\eqref{eq:qterrormodel} becomes \emph{exact} for a randomized FP4 quantizer $2^{U} Q_{\mathrm{FP4}}\left(2^{-U}y \right)$ where  $U\sim\Unif([0,1))$ is statistically independent of $y$. We refer to this randomized FP4 quantizer as the \emph{multiplicative-dither} FP4 quantizer, and with some abuse of notation also denote this quantizer as $\Qt(\cdot)$. 
Figure~\ref{fig:phi-fp4} depicts the function $\phi(y)$ as well as   $\phi_{\mathrm{exact}}(y)=\frac{(y-Q_{\mathrm{FP4}}(y))^2}{y^2}$. From the figure, one can see that $\phi(y)$ is minimized within the interval  $y\in\m{I}_{\FP 4}=[7/4,7]$ and takes the value $\phi(y)=\phi_{\mathrm{min}}\approx 0.010723$ there. This is proved in Lemma~\ref{lem:randFP4props} below. In fact, for $y\in\m{I}_{\FP 4}$ it can be shown that $(y-2^{U} Q_{\mathrm{FP4}}\left(2^{-U}y\right))\indep y$, further justifying the model~\eqref{eq:qterrormodel} for the multiplicative-dither FP4 quantizer.
The relative mean squared quantization error (MSE) behaves as in the IEIN model within the interval $\m{I}_{\FP 4}$, and increases as we move away from the boundaries of $\m{I}_{\FP 4}$.

\begin{figure}[t]
	\centering
	\includegraphics[width=0.63\linewidth]{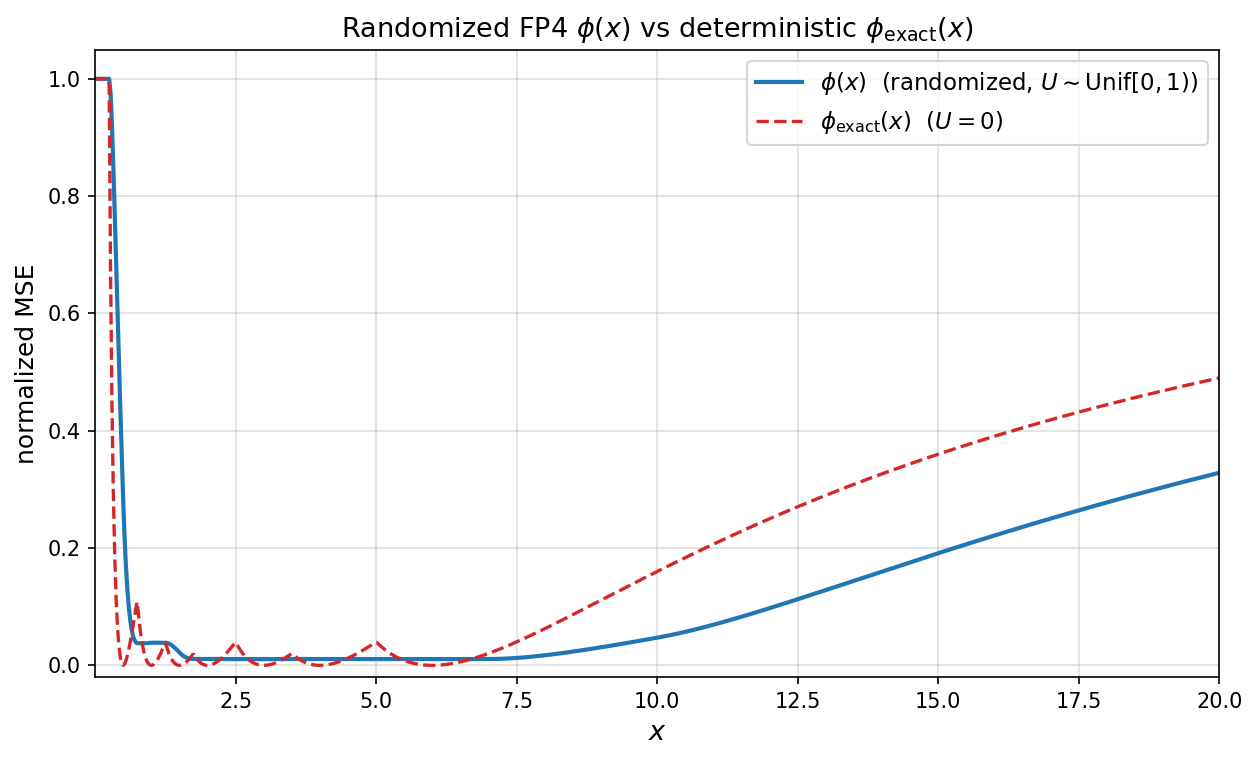}
	\caption{Normalized mean-squared error of randomized FP4 quantization,
		\(\phi(x)=\frac{\mathbb{E}[(x-\tilde Q_{\mathrm{FP4}}(x))^2]}{x^2}\),
		and its deterministic counterpart,
		\(\phi_{\mathrm{exact}}(x)=\frac{(x-Q_{\mathrm{FP4}}(x))^2}{x^2}\).}
	\label{fig:phi-fp4}
\end{figure}

When all entries in a tensor to be FP4-quantized are within $\m{I}_{\FP 4}$, the relative mean squared quantization error will be $\phi_{\mathrm{min}}$ for all of them. However, this is rarely the case in tensors corresponding to LLM weights and activation matrices. Thus, to combat the effect of limited dynamic range, FP4 is usually applied with microscaling as in MXFP4 or NVFP4, where a group of $g$ entries are also given a common scale $\gamma$. In MXFP4 $g=32$ and $\gamma$ is represented in E8M0 format, whereas in NVFP4 $g=16$ and $\gamma$ is represented in E4M3 format. Here we focus on NVFP4, which is becoming a common choice for models running on NVIDIA hardware, though some of the techniques we develop may also be useful for MXFP4 as well. Microscaling allows to attain the minimal relative mean squared quantization error for groups where the ratio between the largest magnitude entry and the smallest magnitude entry is at most $\frac{7}{7/4}=4$. With groups of size $g=16$, even this condition rarely holds in LLMs, and the resulting relative quantization error is greater than $\phi_{\mathrm{min}}$.

To mitigate the in-group dynamic range problem, this work develops an NVFP4 quantization method that utilizes a diagonal matrix $A\in\RR^{n\times n}$ for ``smoothing'' the weight matrix $W\in\RR^{m\times n}$ to $WA^{-1}$ and the activation matrix $X\in\RR^{n\times k}$ to $AX$, as in SmoothQuant~\cite{xiao2023smoothquant}. Roughly speaking, here the goal of the smoothing matrix $A$ is to make all groups of $g=16$ elements of $W A^{-1}$ and of $A X$ that share the same scale as concentrated as possible, such that we can find a scale that places all of them within $\m{I}_{\FP 4}$. 
The clear advantage of diagonal $A$ with respect to more general structures for $A$, e.g. block-rotation~\cite{microscaling}, PCA directions aligning~\cite{chen2025wush}, tensor of dense small matrices~\cite{sun2025flatquant}, is that diagonal pre-processing is cheap and can easily be done in runtime for the activations. In fact, in some layers it can be fused into the LayerNorm or RMSnorm. Our smoothed NVFP4 quantization scheme is illustrated in Figure~\ref{fig:SchemeIllustration}.

\begin{figure}[h]
    \centering
    \includegraphics[width=0.95\linewidth]{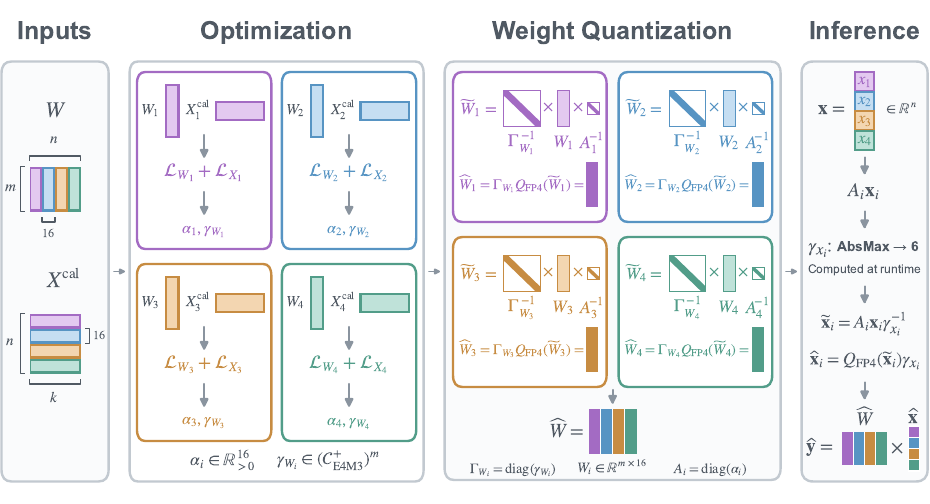}
    \caption{Overview of OSFP4. Each color denotes a group of 16 channels. Using a calibration set, we jointly optimize the diagonal transformation and weight scales, to minimize the corresponding reconstruction losses. The optimized parameters are used for weight quantization (the figure illustrates RTN weight quantization, but OSFP4 also supports SIC weight quantization). At inference, the diagonal transformation is applied to the input before activation scales are computed and activations are quantized.}
    \label{fig:SchemeIllustration}
\end{figure}

In Section~\ref{sec:lossfunc} we derive a first-order approximation for the MatMul quantization loss $L=\EE\|WX -\widehat{WX}\|_F^2$ based on the model~\eqref{eq:qterrormodel}.  It turns out that the loss decouples to $n/16$ independent terms, indexed by $\ell=1,\ldots,n/16$, each corresponding to groups of $16$ consecutive rows of $W$ and columns of $X$. Denote by $W_\ell\in\RR^{m\times 16}$ and $X_\ell\in\RR^{16\times k}$ the sub-matrices of $W$ and $X$, respectively, that participate in the $\ell$th group. Furthermore, denote the corresponding diagonal smoothing parameters, weight group scales, and activation group scales, by $\{\alpha^{(r)}_\ell\}$,  $\{\gamma^{(i)}_{W_{\ell}}\}_{i=1}^m$  and $\{\gamma^{(j)}_{X_\ell}\}_{j=1}^k$, respectively. Our approximate loss expression takes the form
\begin{align}
L&=\sum_{\ell=1}^{n/16} c_\ell L_\ell\\
L_\ell&= \sum_{r=1}^{16}\left[\sum_{i=1}^m p_{\ell}(i,r)\cdot\phi\left(\frac{W_\ell(i,r)}{\gamma_{W_\ell}^{(i)}\alpha^{(r)}_\ell} \right)+  \sum_{j=1}^k \zeta_\ell\cdot q_{\ell}(r,j)\cdot\phi\left(\frac{\alpha^{(r)}_\ell X_\ell(r,j)}{\gamma_{X_\ell}^{(j)}}  \right)  \right]\label{eq:lossapproxintro}
\end{align}	
where $c_\ell,\zeta_\ell,\{p_{\ell}(i,r)\},\{q_{\ell}(r,j)\}$ are all non-negative numbers that depend on $W_\ell$ and $X_\ell$, but \emph{not} on the free parameters $\{\alpha^{(r)}_\ell\}$,  $\{\gamma^{(i)}_{W_{\ell}}\}_{i=1}^m$  and $\{\gamma^{(j)}_{X_\ell}\}_{j=1}^k$. Furthermore, $\sum_{r=1}^{16}\sum_{i=1}^m p_{\ell}(i,r)=1$ and $\sum_{r=1}^{16}\sum_{j=1}^k q_{\ell}(r,j)=1$. Moreover, the coefficients $c_\ell,\zeta_\ell,\{p_{\ell}(i,r)\},\{q_{\ell}(r,j)\}$ characterizing the loss function depend on the rounding method  used for quantizing the scaled and smoothed weights, which can be either round-to-nearest (RTN) or GPTQ-style successive interference cancellation (SIC). See Section~\ref{sec:lossfunc} for more details.

The first term in~\eqref{eq:lossapproxintro} corresponds to the contribution of the weights quantization error to $\EE\|WX -\widehat{WX}\|_F^2$, whereas the second term corresponds to the contribution of the activations quantization error. Since weights are quantized offline (unlike activations that are quantized in runtime), we can use either RTN or SIC for quantizing them.  Typically under SIC, the contribution of the weights quantization error is smaller and we typically have that $\zeta_\ell>1$ in this case. On the other hand, for weight-only quantization, we set $\zeta_\ell=0$.

\begin{figure}[t]
	\centering
	\includegraphics[width=0.9\linewidth]{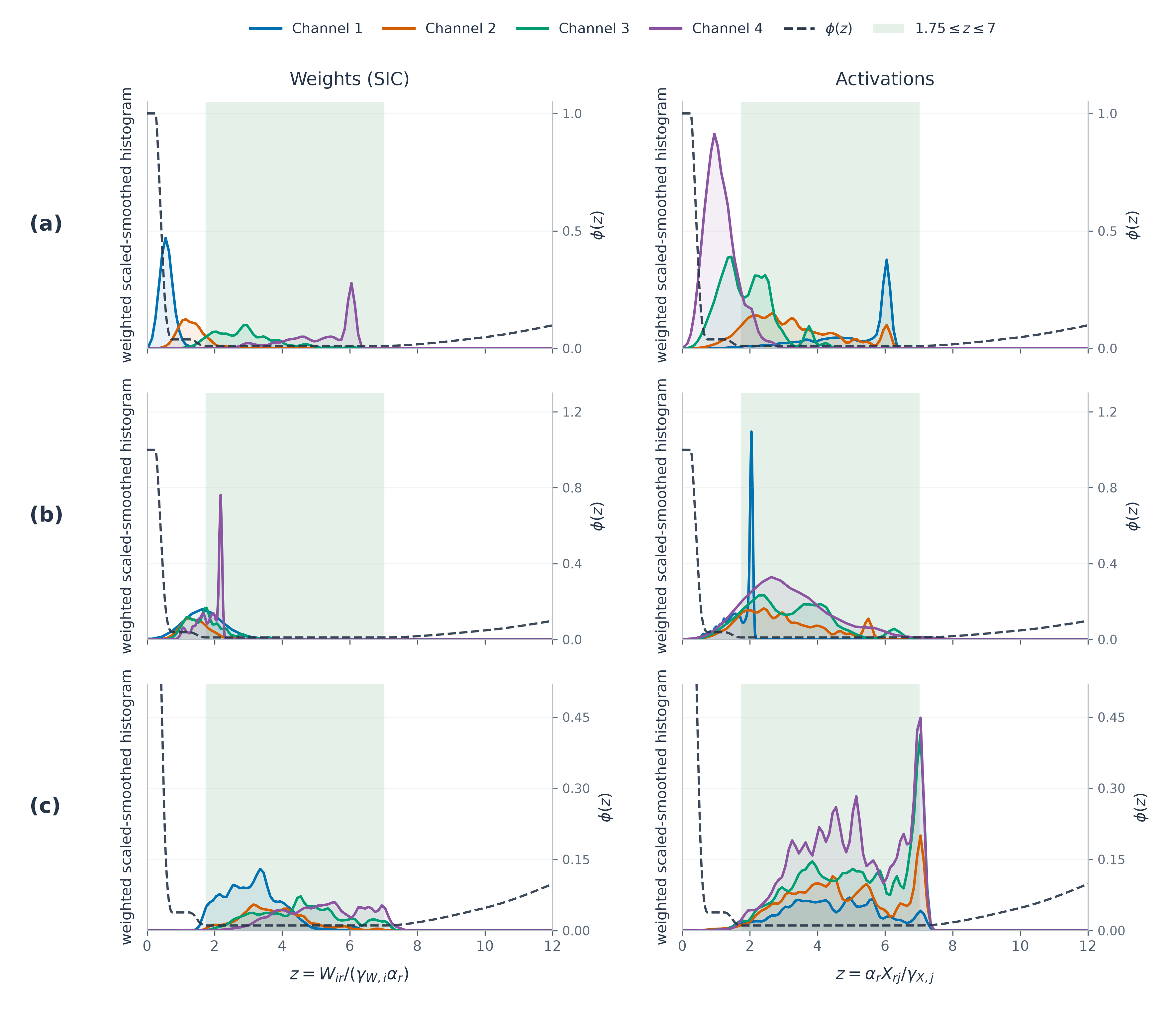}
	\caption{OSFP4 optimization for a synthetic 16-channel group
		using weight SIC (four channels shown). Weighted histograms
		of scaled-smoothed weights (left) and activations (right),
		for $z\geq0$, are overlaid with $\phi(z)$ (dashed).
		(a) $A=I$ and absmax-to-6 scales;
		(b) optimized $A$ with fixed scales (same scales as in (a));
		(c) jointly optimized $\{\alpha^{(r)}_\ell\}$,  $\{\gamma_{W^{(i)}_{\ell}}\}_{i=1}^m$  and $\{\gamma^{(j)}_{X_\ell}\}_{j=1}^k$.}
	\label{fig:osfp4-scheme}
\end{figure}

Our OSFP4 quantization scheme first optimizes the smoothing and scaling coefficients to minimize~\eqref{eq:lossapproxintro}. This is illustrated in Figure~\ref{fig:osfp4-scheme}.  Each channel in the figure corresponds to a different value of $r\in\{1,\ldots,16\}$, and we only plot $4$ channels to avoid overloading the figure with details. Note that the histograms on the right are higher than those on the left. This follows since the figure corresponds to SIC weights quantization, and consequently $\zeta_\ell>1$. Panel (a) corresponds to the ``default'' setup in NVFP4 quantization where no smoothing is used ($A=I$) and weight and activation scales are dictated by the absmax-to-6 rule, which puts the maximal entry of each scaled vector in $\RR^{16}$ at $6$. In panel (b) we keep  $\{\gamma^{(i)}_{W_{\ell}}\}_{i=1}^m$  and $\{\gamma^{(j)}_{X_\ell}\}_{j=1}^k$ with the same value as in (a), and only allow to optimize the diagonal smoothing coefficients $\{\alpha^{(r)}_\ell\}$. This corresponds to stretching/contracting each channel's histogram with inverse factors for the weight and activation. In panel (c) we jointly optimize all free parameters $\{\alpha^{(r)}_\ell\}$,  $\{\gamma^{(i)}_{W_{\ell}}\}_{i=1}^m$  and $\{\gamma^{(j)}_{X_\ell}\}_{j=1}^k$. Joint optimization allows to place more mass in $\m{I}_{\FP 4}$, resulting in smaller total loss.

Unlike weights which are fixed and known, activations are computed online. Consequently, we cannot compute the loss function~\eqref{eq:lossapproxintro}. As usual, we resort to optimizing a ``calibration-based'' loss function where $X$ in~\eqref{eq:lossapproxintro} consists of activations computed from some calibration set. We fix the smoothing coefficients $\{\alpha^{(r)}_\ell\}$ minimizing the calibration-based loss and replace $W$  with $W A^{-1}$ and the runtime activations $X$ with $A X$. For the group scales of $W$ we could use the $\{\gamma^{(i)}_{W_{\ell}}\}_{i=1}^m$  optimizing the loss~\eqref{eq:lossapproxintro}. However, this loss corresponds to the randomized multiplicative-dither FP4 quantizer, whereas in practice we do not use multiplicative dither. Since that the multiplicative dither $2^U$ corresponds to multiplying each $\gamma^{(i)}_{W_{\ell}}$ by a random number in $[1,2)$, we can do better by choosing the optimal E4M3 scale in the range $[\gamma^{(i)}_{W_{\ell}},2\gamma^{(i)}_{W_{\ell}})$, see Lemma~\ref{lem:randFP4props}. We therefore choose the weight group scales using brute-force search within (approximately) this range. For the activations group scales, it makes more sense to rely on the actual runtime values than on the calibration data. We therefore choose them in runtime using the absmax-to-6 rule (with respect to $AX$).

Experiments show that OSFP4 achieves the highest average accuracy among the evaluated competitors in the corresponding quantization settings, while retaining approximately 94-97\% of vendor NVFP4 prefill throughput on the measured workloads.
We provide an \href{https://github.com/neriahbd/OSFP4}{add-on package for LLM Compressor}~\cite{llmcompressor2024} that implements our OSFP4 quantization method and exports quantized models as vLLM checkpoints, together with a \href{https://github.com/neriahbd/vllm-osfp4}{vLLM plugin} that support the additional diagonal smoothing layers.

\textbf{Related Work:} 
Post-training quantization uses calibration data to preserve layer outputs without retraining the original model. GPTQ~\citep{frantar2023gptq} uses second-order information to compensate quantization errors through updates to unquantized weights. This method has demonstrated empirical benefits for various quantization codebooks, including for FP and in particular NVFP4 constellations~\cite{microscaling}.
Diagonal scaling has proved useful for INT constellations: SmoothQuant~\citep{xiao2023smoothquant} balances the quantization difficulty of weights and activations by redistributing their magnitudes, while AWQ~\citep{lin2024awq} uses activation statistics to scale and protect important weight channels. OmniQuant~\citep{shao2024omniquant} learns such scalings as part of a broader family of equivalent transformations, jointly with clipping parameters, through reconstruction objectives. These methods establish the value of calibration, scaling, and feedback. WaterSIC~\citep{lifar26} combines all three and finds a diagonal smoothing matrix which, when combined with GPTQ/SIC, obtains near-information-theoretic optimal weight-only entropy coded integer quantization. Another recent work~\citep{smoothGPTQ} addresses optimal diagonal smoothing for weights and activations INT quantization using GPTQ is. It turns out that NVFP4 requires a markedly different criterion for choosing the diagonal smoothing matrix.
 OSFP4 jointly adapts smoothing and block scales to the NVFP4 grid and the chosen rounding procedure. Prior work demonstrated the value of randomized rotations for low-bit quantization~\citep{chee2023quip,ashkboos2024quarot,liu2025spinquant}. For NVFP4, however, full-dimensional rotations were found less beneficial and can even hurt RTN accuracy, whereas MR-GPTQ~\citep{microscaling} benefits from rotations within quantization groups together with format-specific GPTQ optimizations. OSFP4 avoids rotations and instead uses cheaper diagonal transformations, jointly optimized (offline) with the block scales. Despite this simpler transformation, our W-SIC/X-RTN configuration achieves higher average W4A4 accuracy than MR-GPTQ under the common evaluation protocol in Table~\ref{tab:accuracy-summary}.
Recent work shows that absmax scaling leaves substantial room for improvement. Four Over Six (4/6)~\citep{cook2026fouroversix} adaptively selects between absmax-to-4 and absmax-to-6 scaling for each block. ScaleSearch~\citep{gupta2026search} searches block scales to reduce reconstruction error, and ScaleSweep~\citep{lin2026scalesweep} derives bounded search ranges for MSE and weighted MSE objectives. SOAR~\citep{bao2026soar} combines joint global/block-scale optimization with decoupled encoding and decoding scales. H-Scale~\citep{yu2026hscale} uses calibration activation second moments to select E4M3 block scales by minimizing a weighted-mean-squared weight reconstruction error. OSFP4 combines accurate scale selection with diagonal smoothing, jointly optimizing both while accounting for the applied rounding procedure (RTN/SIC).  Online activation-scale search is complementary to this procedure; our default uses absmax activation scaling. Our analysis builds on high-rate models of floating-point error and covariance-aware quantized matrix multiplication~\citep{OP26surveyI,OP26surveyII}. 

\section{Loss Functions}
\label{sec:lossfunc}

\subsection{Smoothed FP4 quantization error} 
Denote by 
\begin{align}
\Cfp=\left\{0,\pm\frac{1}{2},\pm 1,\pm\frac{3}{2},\pm 2,\pm 3,\pm 4,\pm 6 \right\}
\label{eq:CscaleDef}
\end{align}
the points in the FP4 E2M1 constellation, and for $y\in\RR$ let
\begin{align}
Q_{\FP 4}(y)=\argmin_{\hat{y}\in\Cfp}|y-\hat{y}|    
\end{align}
be the deterministic nearest-neighbor quantizer of $y$ to FP4. The quantization error $y-Q_{\FP 4}(y)$ is a deterministic and non-smooth function of $y$. Consequently, it is hard to develop an intuitive understanding of what pre-processing of the matrices $W,X$ to be FP4-quantized to $\hat{W},\hat{X}$ will result in a small $\|WX-\hat{W}\hat{X}\|_F^2$. To that end, instead of analyzing the exact error of the quantizer, we develop a stochastic approximate model for it. This model is simple enough to guide the development of FP4-optimized quantization schemes, and is found to be remarably accurate in our numerical experiments.

More concretely, inspired by the \emph{dithered absmax FP quantizer} from~\citet{OP26surveyI}, instead of analyzing the deterministic quantizer $Q_{\FP 4}(y)$  we consider the following randomized multiplicative-dither FP4 quantizer 
\begin{align}
\tilde{Q}_{\mathrm{FP4}}(y)=2^{U} Q_{\mathrm{FP4}}\left(2^{-U}y \right),
\label{eq:QtDef}
\end{align}
where $U\sim\Unif([0,1))$ is statistically independent of $y$, and its corresponding normalized MSE function $\phi(y)$ is defined in~\eqref{eq:PhiDef}. We then simplify our model even further, by 
modeling the randomized multiplicative-dither FP4 quantizer~\eqref{eq:QtDef} using the additive noise channel~\eqref{eq:qterrormodel}, with $Z\indep y$. We further assume that the random quantization noises $Z$ for different coordinates are iid. The simplified additive noise channel~\eqref{eq:qterrormodel} has the exact same MSE as the multiplicative-dither FP4 quantizer~\eqref{eq:QtDef}. However, its induced quantization noise $y-2^{U} Q_{\mathrm{FP4}}\left(2^{-U}y\right)$ is not statistically independent of $y$, and is not even zero-mean. Nevertheless, this models leads to a simple analysis of the MatMul quantization loss, and as we shall see in the numerical experiments, the resulting loss predictions are remarkably accurate.


The following lemma provides two key properties of $\tilde{Q}_{\mathrm{FP4}}(y)$ and $\phi(y)$.
\begin{lemma}
\begin{enumerate}
\item For any $y\in\RR^n$ there exists a fixed $\beta\in [1,2)$ such that
\begin{align}
\|y-\beta \cdot Q_{\mathrm{FP4}}(\beta^{-1} y)\|^2\leq \sum_{i=1}^n y_i^2 \phi(y_i).   
\end{align}
\item $\phi(y)$ is an even function. Its minimum is $\phi_{\mathrm{min}}=\EE[(1-2^{-(U+1)}\mathrm{round}(2^{U+1}))^2]\approx 0.010723$ and it attains this minimum on the entire interval $\m{I}_{\mathrm{FP4}}=[7/4,7]$. 
\end{enumerate}
\label{lem:randFP4props}
\end{lemma}
A proof is given in Appendix~\ref{app:randomized-fp4-proof}. The plateau in $\m{I}_{\mathrm{FP4}}$ is visibly clear in Figure~\ref{fig:phi-fp4}. The interpretation of the first part of the lemma is that for any vector $y\in\RR^n$ a deterministic FP4 quantizer can always achieve smaller MSE than the multiplicative-dither FP4 quantizer by pre/post scaling the input/output with some $\beta\in[1,2)$. Since NVFP4 involves pre/post scaling anyways, this $\beta$ factor can be embedded within the corresponding group scales. Consequently, the MSE of the multiplicative-dither FP4 quantizer serves as an upper bound on the MSE of deterministic NVFP4 quantizers.

\subsection{Matrix-product error under RTN and SIC}

Let $W\in\RR^{m\times n}$ and $X\in\RR^{n\times k}$, where the columns of $X$ are calibration activations, and define $\Sigma_X=XX^\top/k$. We partition the inner dimension into $L=n/16$ groups and define the sub-matrices: 
\begin{align*}
W_\ell=W(:,1+16(\ell-1):16\ell)\in\RR^{m\times 16}~~ \text{and}~~X_\ell=X(1+16(\ell-1):16\ell,:)\in\RR^{16 \times k},    
\end{align*}
for $\ell=1,\ldots,L$. See Figure~\ref{fig:SchemeIllustration} for an illustration with $L=4$. 
Fix the corresponding diagonal smoothing parameters $\left\{\alpha_\ell^{(r)}\right\}_{r=1}^{16}$, the weight row scales $\Gamma_{W_\ell}=\left\{\gamma_{W_\ell}^{(i)}\right\}_{i=1}^m$, and activation column scales $\Gamma_{X_\ell}=\left\{\gamma_{X_\ell}^{(j)}\right\}_{j=1}^k$. Define $A_\ell=\diag(\alpha_\ell^{(1)},\ldots,\alpha_\ell^{(16)})$, $\Gamma_{W_\ell}=\diag\left(\gamma_{W_\ell}^{(1)},\ldots,\gamma_{W_\ell}^{(m)} \right)$ and $\Gamma_{X_\ell}=\diag\left(\gamma_{X_\ell}^{(1)},\ldots,\gamma_{X_\ell}^{(k)} \right)$.

Let
\begin{align}
\tilde{W}_\ell=\Gamma^{-1}_{W_\ell}W_\ell A_\ell^{-1}\in\RR^{m\times 16},\qquad
\tilde{X}_\ell=A_\ell X_\ell \Gamma^{-1}_{X_\ell}\in\RR^{16 \times k}
\label{eq:scaledMatsIntro}
\end{align}
denote the smoothed and scaled matrices to be quantized, and denote their FP4-quantized versions by $\hat{W}_\ell^{\FP4}\in\Cfp^{m\times 16}$ and $\hat{X}_\ell^{\FP4}\in\Cfp^{16\times k}$, respectively. The rounding method used for mapping $\tilde{W}$ to $W^{\FP4}$ and $\tilde{X}$ to $X^{\FP4}$ will be discussed below. Finally, denote the de-quantized matrices by 
\begin{align}
 \hat{W}_\ell= \Gamma_{W_\ell} \hat{W}_\ell^{\FP4},~~~ \hat{X}_\ell= \hat{X}_\ell^{\FP4}\Gamma_{X_\ell} 
\end{align}
and the corresponding quantization errors as
\begin{align}
e_{W_\ell}=\hat{W}_\ell-(W_\ell A^{-1}_\ell),~~~e_{X_\ell}=\hat{X}_\ell-(A_\ell X_\ell).
\end{align}
Setting 
\begin{align}
e_W=[e_{W_1}|\cdots|e_{W_L}],~~~e_X= \left[\begin{array}{@{}c@{}}
e_{X_1} \\
\hline
\vdots \\
\hline
e_{X_L}
\end{array}\right],    
\end{align}
we obtain that the MatMul quantization error is
\begin{align}
e=\hat W\hat X-WX=e_WAX+WA^{-1}e_X+e_We_X
\approx e_WAX+WA^{-1}e_X,
\label{eq:error1storder}
\end{align}
where in the last approximation we neglected the second-order term $e_We_X$. For weight-only quantization we set $e_X=0$ and obtain $e=e_W AX$.



We proceed to approximate the MSE $\EE\|e\|_F^2$ using the additive noise channel model~\eqref{eq:qterrormodel}. The MSE depends on the rounding method that is used, and we distinguish between round-to-nearest (RTN) and successive interference cancellation (SIC). The complete derivations are given in Appendix~\ref{app:loss-derivations}.

\textbf{Round to nearest (RTN) quantization:} Under RTN quantization we simply set
\begin{align}
\hat{W}_\ell^{\FP4}= \Qt\left(\tilde{W}_\ell\right),~~~\hat{X}^{\FP4}_\ell= \Qt\left(\tilde{X}_\ell\right).
\end{align}
A routine computation, detailed in Appendix~\ref{app:loss-derivations} gives
\begin{align}
 L^{\mathrm{RTN}}_W\triangleq \EE\|e_W A X\|_F^2=\sum_{\ell=1}^L L^{\mathrm{RTN}}_{W_\ell}(\Gamma_{W_\ell},A_\ell),  
\end{align}
where
\begin{align}
L^{\mathrm{RTN}}_{W_\ell}(\Gamma_{W_\ell},A_\ell)=\EE\|e_{W_\ell}A_\ell X_\ell\|_F^2=k\sum_{r=1}^{16}\sigma^2_{X_\ell}(r)\sum_{i=1}^m W^2_{\ell}(i,r) \cdot\phi\left(\frac{W_\ell(i,r)}{\gamma_{W_\ell}^{(i)}\alpha^{(r)}_\ell} \right),
\label{eq:LWRTN}
\end{align}
and $\sigma^2_{X_\ell}(r)=\frac{1}{k}\sum_{j=1}^k X^2_\ell(r,j)$ (we also have $\sigma^2_{X_\ell}(r)=\Sigma_X(r+16(\ell-1),r+16(\ell-1))$).    

Similarly, when activations are also quantized using RTN, we obtain
\begin{align}
 L^{\mathrm{RTN}}_X=\EE\|W A^{-1} e_X\|_F^2=\sum_{\ell=1}^L L^{\mathrm{RTN}}_{X_\ell}(\Gamma_{X_\ell},A_\ell).      
\end{align}
where
\begin{align}
L^{\mathrm{RTN}}_{X_\ell}(\Gamma_{X_\ell},A_\ell)=\EE\|W_\ell A^{-1}_\ell e_{X_\ell}\|_F^2=m\sum_{r=1}^{16}\sigma^2_{W_\ell}(r)\sum_{j=1}^k X^2_{\ell}(r,j) \cdot\phi\left(\frac{\alpha^{(r)}_\ell}{\gamma_{X_\ell}^{(j)}} X_\ell(r,j) \right) ,   
\label{eq:LXRTN}
\end{align}
and $\sigma^2_{W_\ell}(r)=\frac{1}{m}\sum_{i=1}^m W^2_{\ell}(i,r)$. 

Both the expressions for $L^{\mathrm{RTN}}_W$ and for $L^{\mathrm{RTN}}_X$ are exact for RTN quantization using the additive noise channel model~\eqref{eq:qterrormodel} for $\Qt$.

\textbf{Successive interference cancellation (SIC):} Note that
\begin{align}
 \|e_W AX\|_F^2=\tr\left((e_W AX)(e_W AX)^\top\right)=k\tr\left(e_W A\Sigma_X A^\top e_W^\top \right).   
\end{align}
Let $U\in\RR^{n\times n}$ be an upper triangular matrix such that $\Sigma_X=U^\top U$, and denote 
\begin{align}
Y=UA^\top (WA^{-1})^\top=U W^\top\in \RR^{n\times m}.
\end{align}
We have
\begin{align}
 \|e_W AX\|_F^2=k\|U A^\top e_W^\top\|_F^2=k\|Y-UA\hat{W}^\top\|_F^2.   
 \label{eq:WMSEloss}
\end{align}
RTN quantization aims to minimize $\|(WA^{-1})^\top-\hat{W}^\top\|_F^2$. However, as evident from~\eqref{eq:WMSEloss} this is sub-optimal as it ignores $\Sigma_X$, which dictates the loss function through its Cholesky factor $U$. A smaller loss can be obtained by taking $U$ into account in the quantization procedure. 
A commonly used low-complexity algorithm for exploiting the structure of $U$ in order to reduce the MSE is the SIC/GPTQ approach. In SIC weight quantization the weight columns are quantized successively, starting from the last. After every column is quantized the corresponding quantization error is fed back to the remaining columns, so that these errors can be compensated for when quantizing the next columns. The full procedure is detailed in Algorithm~\ref{alg:NVFP4SIC} in Appendix~\ref{app:loss-derivations}. The analysis of the resulting MSE, which is carried out in Appendix~\ref{app:loss-derivations}, gives
\begin{align}
L_W^{\mathrm{SIC}}=k\EE\|Y-UA\hat{W}^\top\|_F^2,
\end{align}
where
\begin{align}
L^{\mathrm{SIC}}_{W_\ell}(\Gamma_{W_\ell},A_\ell)
=k\sum_{r=1}^{16}U_\ell^2(r)\sum_{i=1}^mW_\ell^2(i,r)
\phi\left(\frac{W_\ell(i,r)}{\gamma_{W_\ell}^{(i)}\alpha_\ell^{(r)}}\right),
\label{eq:LWSIC}
\end{align}
and $U_\ell(r)=U(16(\ell-1)+r,16(\ell-1)+r)$. The only difference from $L^{\mathrm{RTN}}_{W_\ell}(\Gamma_{W_\ell},A_\ell)$ is replacing $\sigma^2_{X_\ell}(r)$ by $U_\ell^2(r)$. The latter is the variance of the optimal linear prediction error of coordinate $t=16(\ell-1)+r$ from coordinates $t+1:n$ of a random vector with covariance $\Sigma_X$, and hence $U_\ell^2(r)\leq\sigma^2_{X_\ell}(r)$. This distinction allows us to optimize $A$ taking into account that SIC will be used for the actual quantization.


\textbf{Total quadratic Distortion:} Using~\eqref{eq:LWRTN}~\eqref{eq:LXRTN} and~\eqref{eq:LWSIC} we obtain the following loss functions for randomized NVFP4 quantization, as a function of $A$ and the E4M3 scales, under the various quantization setups:
\begin{align}
L^{W-\mathrm{RTN}}&=\sum_{\ell=1}^L L^{\mathrm{RTN}}_{W_\ell}(\Gamma_{W_\ell},A_\ell)\label{eq:PrdictedLossWRTN}\\
L^{W-\mathrm{SIC}}&=\sum_{\ell=1}^L L^{\mathrm{SIC}}_{W_\ell}(\Gamma_{W_\ell},A_\ell)\label{eq:PrdictedLossWSIC}\\
L^{W-\mathrm{RTN},X-\mathrm{RTN}}&=\sum_{\ell=1}^L \left[L^{\mathrm{RTN}}_{W_\ell}(\Gamma_{W_\ell},A_\ell)+L^{\mathrm{RTN}}_{X_\ell}(\Gamma_{X_\ell},A_\ell)\right]\label{eq:PrdictedLossWRTNXRTN}\\
L^{W-\mathrm{SIC},X-\mathrm{RTN}}&=\sum_{\ell=1}^L \left[L^{\mathrm{SIC}}_{W_\ell}(\Gamma_{W_\ell},A_\ell)+L^{\mathrm{RTN}}_{X_\ell}(\Gamma_{X_\ell},A_\ell)\right].
\label{eq:PrdictedLossWSICXRTN}
\end{align}
Note that in all expressions the dependency on $A$ and the E4M3 scales \emph{completely decouples to $L$ different terms, that depend on different parameters}. Therefore, the optimization on $\{A_\ell,\Gamma_{W_\ell},\Gamma_{X_\ell}\}_{\ell=1}^L$ decouples to $L$ independent optimization problems. In Appendix~\ref{app:loss-derivations} we specify how to obtain the form of~\eqref{eq:lossapproxintro} from the loss expressions above.


\section{The OSFP4 Quantization Algorithm}
\label{sec:optimization}

We are now ready to describe how the Optimized Smoothing and Scaling for NVFP4 (OSFP4) quantization algorithm works. This algorithm has four different configurations: $W-\mathrm{RTN},W-\mathrm{SIC},W-\mathrm{RTN}/X-\mathrm{RTN}$ or $W-\mathrm{SIC}/X-\mathrm{RTN}$, corresponding to either W4A16 or W4A4, and to whether $W$ is quantized using RTN or using SIC. The parameters for each linear projection $WX$ are optimized independently. We describe the procedure for a single projection operation. Our quantization pipeline is as follows.

\textbf{Calibration set:} For all $4$ configurations we first collect $k_1$ activation vectors $x_1,\ldots,x_{k_1}$ and set $\Sigma_X=\frac{1}{k_1}\sum_{j=1}^{k_1}x_j x_j^\top$. Since computing $\Sigma_X$ does not require many resources even when $k_1$ is large, we typically use large $k_1$, as is standard in calibration based quantization. The terms $L^{W-\mathrm{RTN}},L^{W-\mathrm{SIC}}$ depend on the calibration set only through $\Sigma_X$, and for weight-only quantization we therefore keep only $\Sigma_X$ and  set $k=k_1$. In the setups involving activations quantization, the term $L^{X-\mathrm{RTN}}$ also participates in the loss function, and the complexity of optimizing this term grows with $k$. Even though the optimization is performed offline, we will limit $k$ to modest values, such that the quantization procedure will not require too many resources. We therefore take $k<k_1$ and  select a random subset $\m{J}=\{j_1,\ldots,j_k\}$ of  $k$ calibration vectors and set $X=[x_{j_1}|\cdots|x_{j_k}]\in\RR^{n\times k}$. 

\textbf{Optimization of $A$ (and initial guess for $\Gamma_W)$:} Depending on the configuration we construct the relevant loss function from ~\eqref{eq:PrdictedLossWRTN}-\eqref{eq:PrdictedLossWSICXRTN} (note that the first two depend only on $\Sigma_X$, whereas the last two also depend on $X$). Then, for every $\ell=1,\ldots,L$ we perform joint optimization on $A_\ell,\Gamma_{W_\ell}$ and also $\Gamma_{X_\ell}$ in the W4A4 configurations. As described above, all loss functions \eqref{eq:PrdictedLossWRTN}-\eqref{eq:PrdictedLossWSICXRTN} decouple to $L$ independent terms, so optimizing for each $\ell$ separately is optimal. The optimization is iterative and is initialized with $A=I_n$ and the $\Gamma_W,\Gamma_X$ corresponding to absmax-to-6 scaling. The precise details of the optimization iterations are given in Appendix~\ref{app:scale-optimization}. To illustrate the benefit of this optimization over the standard initialization values, Figure~\ref{fig:Limprovement} in Appendix~\ref{app:scale-optimization} shows how much $L^{W-\mathrm{SIC},X-\mathrm{RTN}}$ improved from its initial value along the different layers of Llama-3.1-8B-Instruct. The value of $A$ found by optimization is the final value used for quantization. The value found for $\Gamma_W$ will be used as an initial guess to dictate the search space for the next step. Even in the setups W4A4 configurations, the value of $\Gamma_X$ obtained by minimizing the relevant loss from ~\eqref{eq:PrdictedLossWRTN}-\eqref{eq:PrdictedLossWSICXRTN} is not used for the actual activations quantization. Instead absmax-to-6 scaling is computed in runtime based on the actual activation vectors (rather than calibration).

\textbf{Optimization of $\Gamma_W$:} The value of $\Gamma_W$ found in the previous step was optimized assuming the randomized quantizer $\Qt$  is used, whereas in practice we use the deterministic quantizer $Q_{\FP 4}(\cdot)$. 
Nevertheless, using part 1 of Lemma~\ref{lem:randFP4props}, for any fixed $y\in\RR^{16}$ and $\gamma_{W_\ell}^{(i)}$ there must be some $\beta_{W_\ell}^{(i)}\in [1,2)$ such that the distortion for the deterministic $\gamma_{W_\ell}^{(i)}\beta_{W_\ell}^{(i)}Q_{\FP 4}\left({y}/(\gamma_{W_\ell}^{(i)}\beta_{W_\ell}^{(i)}) \right)$ is no larger than the expected distortion of the randomized quantizer $\gamma_{W_\ell}^{(i)}\Qt\left({y}/{\gamma_{W_\ell}^{(i)}} \right)$. Thus, there is a deterministic scale $\gamma'=\gamma_{W_\ell}^{(i)}\beta_{W_\ell}^{(i)}$ in the interval $[\gamma_{W_\ell}^{(i)},2\gamma_{W_\ell}^{(i)}]$, for which the deterministic quantizer $\gamma' Q_{\FP 4}\left({y}/{\gamma'} \right)$ attains smaller or equal distortion than that of the randomized quantizer $\gamma_{W_\ell}^{(i)}\Qt\left({y}/{\gamma_{W_\ell}^{(i)}} \right)$. In fact, the distortion for the optimal $\gamma'$ may be substantially smaller than that of the randomized quantizer, as it can always place at least one of the entries of $y$ on a point in $\Cfp$, canceling its quantization error completely. Consequently, for each $\ell=1,\ldots,L$ and $i=1,\ldots,m$ we  determine the final weight scale by trying all valid E4M3 scales in an interval slightly greater than $[\gamma_{W_\ell}^{(i)},2\gamma_{W_\ell}^{(i)}]$. The details are given in Appendix~\ref{app:projection-mse}.


\textbf{Optimization of $\Gamma_X$:} We first need to multiply the activation $X$ by the diagonal matrix $A$. This is handled in the LayerNorm whenever it exists, or via an additional diagonal scaling layer when there is no LayerNorm. Under W4A16 configurations the activation $X$ is not quantized. Under W4A4 the activations need to be quantized in inference time. Our default is to compute a global scale for $X$ based on the calibration set, and then quantize each $16$ consecutive entries using absmax-to-6. The fixed global scale allows to perform this quantization online, even before the entire tensor of activations was computed. This is the default choice in our experimental results. We note in passing that recent work~\cite{gupta2026search,lin2026scalesweep} advocates searching a small bank of E4M3 values around the absmax-to-6 value even for the activations, and shows that the overhead this search incurs on the quantization time is quite reasonable~\cite[Table 6]{gupta2026search}. This optimization is referred to as ScaleSearch/ScaleSweep.

\section{Experimental Results}
\label{sec:experiments}

We evaluate whether OSFP4 improves NVFP4 accuracy, how its components contribute, and whether these gains retain the efficiency of native FP4 inference. We focus here on Llama-3.1-8B-Instruct; Appendix~\ref{app:additional-models} contains results for additional models. 

\textbf{Accuracy under a common evaluation protocol.}
Table~\ref{tab:accuracy-summary} summarizes the mean accuracy on MMLU-CoT, GSM8K, HellaSwag, and WinoGrande. 
Among all W4A4 schemes that were evaluated, OSFP4 W-SIC/X-RTN attained the highest average. Similarly, for W4A16, OSFP4 W-SIC attained the highest average. Full task results, calibration settings, and implementation details are given in Appendix~\ref{app:native-benchmark}. Online activation-scale search is complementary to OSFP4, and was evaluated mainly for enabling a fair comparison with ScaleSweep~\citet{lin2026scalesweep}. Note that even with absmax activation scaling OSFP4 W-SIC/X-RTN scores better than ScaleSweepMSE+GPTQ, whereas the latter is less efficient as it optimizes the activation scale online.


In Appendix~\ref{app:additional-models} we further evaluate OSFP4 for Qwen3-8B, Mistral-7B-Instruct-v0.3, OLMo-3-7B-Instruct, Qwen3-30B-A3B-Instruct (a mixture of experts model), and Gemma-4-31B-IT. For all these models OSFP4 attained the highest average among all compared schemes in the W4A4 setup, and for W4A16 it attained the highest average for most models. For some of the models above we have found that the OSFP4 W-RTN configuration outperformed the OSFP4 W-SIC configuration. This is not a special merit of OSFP4: For some of the models we tested the RTN benchmark sometimes performs better than the GPTQ benchmark.

\begin{figure}[t]
\centering
\begin{minipage}[t]{0.49\linewidth}
\vspace{0pt}
\captionsetup{font=small,skip=5pt}
\centering
\footnotesize
\captionof{table}{Mean accuracy (\%) on Llama-3.1-8B-Instruct. Full scores: Table~\ref{tab:native-nvfp4-comparison}.}
\label{tab:accuracy-summary}
\setlength{\tabcolsep}{3pt}
\renewcommand{\arraystretch}{0.95}
\begin{tabular}{@{}lr@{}}
\toprule
\textbf{Method} & \textbf{Avg.} \\
\midrule
\multicolumn{2}{@{}l}{\textit{BF16 baseline}} \\
BF16 & \textbf{79.22} \\
\midrule
\multicolumn{2}{@{}l}{\textit{W4A4: absmax activation scaling}} \\
RTN (FP-Quant) & 75.65 \\
GPTQ (FP-Quant) & 76.34 \\
MR-GPTQ (FP-Quant) & 76.18 \\
SOAR & 76.60 \\
H-Scale (W4A4 extension) & 76.67 \\
NVIDIA released checkpoint$^\dagger$ & 76.11 \\
OSFP4 W-RTN / X-RTN & 76.36 \\
OSFP4 W-SIC / X-RTN & \textbf{77.03} \\
\midrule
\multicolumn{2}{@{}l}{\textit{W4A4: online activation-scale search}} \\
ScaleSweepMSE & 76.62 \\
ScaleSweepMSE + GPTQ & 76.69 \\
4over6 & 76.34 \\
OSFP4 W-RTN / X-ScaleSweepMSE & 76.96 \\
OSFP4 W-SIC / X-ScaleSweepMSE & \textbf{77.09} \\
\midrule
\multicolumn{2}{@{}l}{\textit{W4A16: weights only}} \\
RTN & 77.67 \\
H-Scale & 78.00 \\
OSFP4 W-RTN & 78.15 \\
OSFP4 W-SIC & \textbf{78.42} \\
\bottomrule
\end{tabular}
\par\vspace{2pt}
\begin{minipage}{\linewidth}\footnotesize
$^\dagger$Separately released checkpoint; preparation is uncontrolled. Bold: best quantized average in each group.
\end{minipage}

\end{minipage}\hfill
\begin{minipage}[t]{0.49\linewidth}
\vspace{0pt}
\captionsetup{font=small,skip=5pt}
\centering
\footnotesize
\captionof{table}{WikiText-2 perplexity ($\downarrow$). W4A4 uses absmax activation scales.}
\label{tab:ablation-summary}
\setlength{\tabcolsep}{3pt}
\begin{tabular}{@{}lrr@{}}
\toprule
\textbf{Configuration} & \textbf{W4A16} & \textbf{W4A4} \\
\midrule
BF16 reference & 7.22 & 7.22 \\
Absmax RTN ($A=I$) & 7.55 & 7.88 \\
MSE scales ($A=I$) & 7.52 & 7.85 \\
OSFP4 W-RTN & 7.42 & 7.77 \\
SIC ($A=I$, absmax) & 7.45 & 7.77 \\
OSFP4 W-SIC & \textbf{7.39} & \textbf{7.67} \\
\bottomrule
\end{tabular}

\par\vspace{12pt}
\centering
\includegraphics[width=0.90\linewidth]{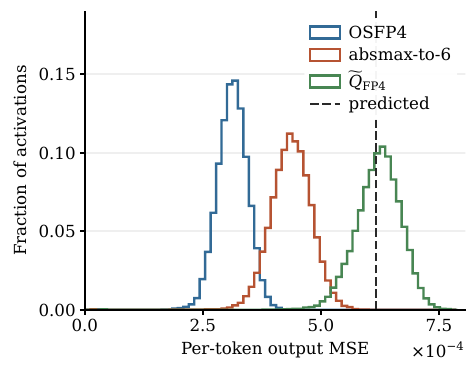}
\captionof{figure}{Layer 12:value, W-RTN/W4A16.} 
\label{fig:projection-mse-panel}
\end{minipage}
\end{figure}

\textbf{Contribution of smoothing, scale selection, and SIC.}
Table~\ref{tab:ablation-summary} compares the full procedure with simpler configurations. Optimizing only the weight scales by unweighted MSE provides a modest perplexity improvement over absmax. OSFP4 W-RTN optimizes $A$ and selects the weight scales using weighted MSE yielding a more significant improvement. Combining the optimized transformation and scale selection with SIC gives the lowest perplexity among the NVFP4 configurations in both precision settings. 
Appendix~\ref{app:component-ablation} further details an ablation study of the MSE for projection-output error across all seven projections, for all variants described here, as well as two close variants of our optimization procedures for $A$, $\Gamma_W$ and $\Gamma_X$. 

\textbf{From predicted loss to deterministic quantization.}
Figure~\ref{fig:projection-mse-panel} examines the accuracy of our MSE loss approximation~\eqref{eq:PrdictedLossWRTN} for the W-RTN configuration on the v-projection in layer $12$. The dashed vertical black line shows the value of~\eqref{eq:PrdictedLossWRTN} for the optimized $A^*,\Gamma^*_W$, whereas the green histogram shows the MSE distribution (over random held-out activations) for OSFP4 W-RTN with the multiplicative-dither quantizer $\Qt$ (which is not practical). The mean error of $\Qt$ agrees with the predicted loss $L^{W-\mathrm{RTN}}(A^*,\Gamma_W^*)$ to within $1\%$
The red histogram shows the MSE distribution for $A^*$-smoothed W-RTN with the deterministic FP4 quantizer applied with absmax-to-6 $\Gamma_W$ scales. Finally, the blue histogram shows the MSE distribution under OSFP4 W-RTN, where $\Gamma_W$ is found by trying all valid E4M3 scales in an interval slightly greater than $[\Gamma^*_W,2\Gamma^*_W]$, as described in Appendix~\ref{app:projection-mse}.
Appendix~\ref{app:projection-mse-experiment} provides details and more extensive results for the other three configurations.

\textbf{Inference efficiency.}
On an NVIDIA B200, OSFP4 retains approximately $94$--$97\%$ of the vendor NVFP4 prefill throughput across the measured workloads. This modest gap is primarily attributable to the explicit diagonal smoothing required before the attention output and MLP down projections, where the transformation cannot be absorbed into a preceding LayerNorm. See Appendix~\ref{app:efficiency}. 

\label{ninepage:main-end}

\newpage
\subsection*{AI use statement}

In this work, we used generative AI tools for assistance with writing and exposition, for writing code and running experiments, and for assistance in implementation of the scale optimization algorithm, as described in Appendix~\ref{app:scale-optimization}.
We have not used generative AI tools for any original idea. 
We take responsibility for the final content of this work,
including text, claims or artifacts produced with the aid of generative AI.

\subsection*{Reproducibility statement}

We provide a complete reproduction package for the OSFP4 experiments
reported in this paper. It contains the experiment scripts, configurations,
software dependency versions, GPU specifications, saved measurements, and
commands for regenerating the reported tables and figures. The package is
available in an
\href{https://anonymous.4open.science/r/reproducibility-osfp4-package/README.md}
{OSFP4 reproduction repository}.

\bibliography{iclr2027_conference}

@inproceedings{ashkboos2024quarot,
  title={{QuaRot}: Outlier-Free 4-Bit Inference in Rotated {LLMs}},
  author={Ashkboos, Saleh and Mohtashami, Amirkeivan and Croci, Maximilian L. and Li, Bo and Cameron, Pashmina and Jaggi, Martin and Alistarh, Dan and Hoefler, Torsten and Hensman, James},
  booktitle={Advances in Neural Information Processing Systems},
  year={2024},
  url={https://arxiv.org/abs/2404.00456}
}

@inproceedings{liu2025spinquant,
  title={{SpinQuant}: {LLM} Quantization with Learned Rotations},
  author={Liu, Zechun and Zhao, Changsheng and Fedorov, Igor and Soran, Bilge and Choudhary, Dhruv and Krishnamoorthi, Raghuraman and Chandra, Vikas and Tian, Yuandong and Blankevoort, Tijmen},
  booktitle={International Conference on Learning Representations},
  year={2025},
  url={https://arxiv.org/abs/2405.16406}
}

@article{OP26surveyII,
  title={High-Rate Quantized Matrix Multiplication {II}},
  author={Ordentlich, Or and Polyanskiy, Yury},
  journal={arXiv preprint arXiv:2605.13768},
  year={2026},
  url={https://arxiv.org/abs/2605.13768}
}

@article{lifar26,
  title={{WaterSIC}: Information-Theoretically (Near) Optimal Linear Layer Quantization},
  author={Lifar, Egor and Savkin, Semyon and Ordentlich, Or and Polyanskiy, Yury},
  journal={arXiv preprint arXiv:2603.04956},
  year={2026},
  url={https://arxiv.org/abs/2603.04956}
}

@article{cook2026fouroversix,
  title={Four Over Six: More Accurate {NVFP4} Quantization with Adaptive Block Scaling},
  author={Cook, Jack and Guo, Junxian and Xiao, Guangxuan and Lin, Yujun and Wyss, Keith and Nazemi, Mahdi and Mishra, Asit and del Mundo, Carlo and Blankevoort, Tijmen and Han, Song},
  journal={arXiv preprint arXiv:2512.02010},
  year={2026},
  url={https://arxiv.org/abs/2512.02010}
}

@unpublished{smoothGPTQ,
  title={{SmoothGPTQ}: Optimizing Diagonal Scaling for {GPTQ} Weight-Activation Quantization},
  author={Netanell Avisdris and Erez Badash and Yoray Herzberg and Or Ordentlich},
  year={2027},
  note={Submitted}
}

@article{pyatkin2025generalizing,
  title = {Generalizing verifiable instruction following},
  author = {Pyatkin, Valentina and Malik, Saumya and Graf, Victoria and Ivison, Hamish and Huang, Shengyi and Dasigi, Pradeep and Lambert, Nathan and Hajishirzi, Hanna},
  journal = {Advances in Neural Information Processing Systems},
  volume = {38},
  year = {2025}
}

@inproceedings{xiao2023smoothquant,
  title={Smoothquant: Accurate and efficient post-training quantization for large language models},
  author={Xiao, Guangxuan and Lin, Ji and Seznec, Mickael and Wu, Hao and Demouth, Julien and Han, Song},
  booktitle={International conference on machine learning},
  pages={38087--38099},
  year={2023},
  organization={PMLR}
}

@article{OP26surveyI,
  title={High-Rate Quantized Matrix Multiplication I},
  author={Ordentlich, Or and Polyanskiy, Yury},
  journal={IEEE BITS the Information Theory Magazine},
  year={2026},
  publisher={IEEE}
}

@article{gupta2026search,
  title = {Search Your Block Floating Point Scales!},
  author = {Gupta, Tanmaey and Prairie, Hayden and Wu, Xiaoxia and Abhyankar, Reyna and Wu, Qingyang and Silveria, Austin and Ponnusamy, Pragaash and Wang, Jue and Athiwaratkun, Ben and Song, Leon and Dao, Tri and Fu, Daniel Y. and De Sa, Chris},
  journal = {Proceedings of Machine Learning and Systems},
  volume = {8},
  pages = {1383--1398},
  url = {https://arxiv.org/abs/2605.12464},
  year = {2026}
}

@inproceedings{microscaling,
	title={Bridging the gap between promise and performance for microscaling FP4 quantization},
	author={Egiazarian, Vage and Castro, Roberto and Kuznedelev, Denis and Panferov, Andrei and Kurtic, Eldar and Pandit, Shubhra and Marques, Alexandre and Kurtz, Mark and Ashkboos, Saleh and Hoefler, Torsten and Alistarh, Dan},
	booktitle={International Conference on Learning Representations},
	volume={2026},
	pages={113529--113563},
	year={2026}
}

@article{chen2025wush,
	title={Wush: Near-optimal adaptive transforms for llm quantization},
	author={Chen, Jiale and Egiazarian, Vage and Castro, Roberto L and Hoefler, Torsten and Alistarh, Dan},
	journal={arXiv preprint arXiv:2512.00956},
	year={2025}
}

@inproceedings{sun2025flatquant,
  title={{FlatQuant}: Flatness Matters for {LLM} Quantization},
  author={Sun, Yuxuan and Liu, Ruikang and Bai, Haoli and Bao, Han and Zhao, Kang and Li, Yuening and Hu, Jiaxin and Yu, Xianzhi and Hou, Lu and Yuan, Chun and Jiang, Xin and Liu, Wulong and Yao, Jun},
  booktitle={Proceedings of the 42nd International Conference on Machine Learning},
  year={2025},
  url={https://arxiv.org/abs/2410.09426}
}

@article{lin2026scalesweep,
  title={{ScaleSweep}: Accurate {NVFP4} Post-Training Quantization of {LLMs} via Block Scale Initialization},
  author={Lin, Li and Wan, Xiaojun},
  journal={arXiv preprint arXiv:2606.07618}, year={2026},
  url={https://arxiv.org/abs/2606.07618}
}

@inproceedings{he2025preserving,
  title={Preserving llm capabilities through calibration data curation: From analysis to optimization},
  author={He, Bowei and Yin, Lihao and Zhen, Hui-Ling and Liu, Shuqi and Wu, Han and Zhang, Xiaokun and Yuan, Mingxuan and Ma, Chen},
  journal={Advances in Neural Information Processing Systems},
  volume={38},
  pages={58531--58572},
  year={2026}
}

@article{bao2026soar,
  title={{SOAR}: Scale Optimization for Accurate Reconstruction in {NVFP4} Quantization},
  author={Bao, Chengzhu and Yan, Xianglong and Li, Zhiteng and Qin, Guangshuo and Yu, Guanghua and Zhang, Yulun},
  journal={arXiv preprint arXiv:2605.12245}, year={2026},
  url={https://arxiv.org/abs/2605.12245}
}

@article{yu2026hscale,
  title={{H-Scale}: Hessian-Guided Scale Refinement for {NVFP4} Sub-Byte {LLM} Inference},
  author={Yu, Hao and Li, Zheng and Liu, Dayiheng and Zhang, Jianwei},
  journal={arXiv preprint arXiv:2608.28113}, year={2026},
  url={https://arxiv.org/abs/2608.28113}
}

@software{llmcompressor2024,
    title={{LLM Compressor}},
    author={Red Hat AI and vLLM Project},
    year={2024},
    month={8},
    url={https://github.com/vllm-project/llm-compressor},
}

@inproceedings{frantar2023gptq,
  title={{GPTQ}: Accurate Post-Training Quantization for Generative Pre-trained Transformers},
  author={Frantar, Elias and Ashkboos, Saleh and Hoefler, Torsten and Alistarh, Dan},
  booktitle={International Conference on Learning Representations},
  year={2023},
  url={https://arxiv.org/abs/2210.17323}
}

@inproceedings{lin2024awq,
  title={{AWQ}: Activation-aware Weight Quantization for {LLM} Compression and Acceleration},
  author={Lin, Ji and Tang, Jiaming and Tang, Haotian and Yang, Shang and Chen, Wei-Ming and Wang, Wei-Chen and Xiao, Guangxuan and Dang, Xingyu and Gan, Chuang and Han, Song},
  booktitle={Proceedings of Machine Learning and Systems},
  year={2024},
  url={https://arxiv.org/abs/2306.00978}
}

@inproceedings{shao2024omniquant,
  title={{OmniQuant}: Omnidirectionally Calibrated Quantization for Large Language Models},
  author={Shao, Wenqi and Chen, Mengzhao and Zhang, Zhaoyang and Xu, Peng and Zhao, Lirui and Li, Zhiqian and Zhang, Kaipeng and Gao, Peng and Qiao, Yu and Luo, Ping},
  booktitle={International Conference on Learning Representations},
  year={2024},
  url={https://arxiv.org/abs/2308.13137}
}

@inproceedings{chee2023quip,
  title={{QuIP}: 2-Bit Quantization of Large Language Models With Guarantees},
  author={Chee, Jerry and Cai, Yaohui and Kuleshov, Volodymyr and De Sa, Christopher},
  booktitle={Advances in Neural Information Processing Systems},
  year={2023},
  url={https://arxiv.org/abs/2307.13304}
}
\bibliographystyle{iclr2027_conference}

\clearpage

\appendix
\section{Proof of Lemma~\ref{lem:randFP4props}}
\label{app:randomized-fp4-proof}

We first express the smoothed relative error as an integral over a unit interval in logarithmic coordinates.

\begin{proposition}[Integral representation of the smoothed relative error]
\label{prop:phi-log-integral}
For every $t\in\RR$,
\begin{align}
 \phi(2^t)
 &=\int_{t-1}^{t}g(v)\,dv,
 &g(v)&=\left(1-2^{-v}Q_{\mathrm{FP4}}(2^v)\right)^2.
 \label{eq:phi-log-integral-proof}
\end{align}
\end{proposition}
\begin{proof}
Substituting~\eqref{eq:QtDef} into the definition of $\phi$ and using $U\sim\Unif([0,1))$ gives
\begin{align}
 \phi(2^t)
 &=\EE_U\left[\left(1-2^{U-t}Q_{\mathrm{FP4}}(2^{t-U})\right)^2\right]\nonumber\\
 &=\int_0^1\left(1-2^{u-t}Q_{\mathrm{FP4}}(2^{t-u})\right)^2\,du.
 \label{eq:phi-dither-expectation}
\end{align}
The substitution $v=t-u$, with $du=-dv$, yields~\eqref{eq:phi-log-integral-proof}. 
\end{proof}

\begin{proof}[Proof of Lemma~\ref{lem:randFP4props}]
The first part follows since the minimum is upper bounded by the expectation, taking $\beta=2^U$. The fact that $\phi(-y)=\phi(y)$  follows immediately from $Q_{\FP 4}(-y)=-Q_{\FP 4}(y)$. It remains to show that $\phi(y)\geq \phi_{\mathrm{min}}$ and that this minimum is attained for all $y\in \m{I}_{\mathrm{FP4}}$. To that end, we first resort to the infinite exponent model of~\citet[Section III.B]{OP26surveyI}. In particular, define the ``$E\infty M 1$'' constellation
\begin{align}
	\mathcal \m{C}_{E\infty M1}
	=\{0\}\cup
	\left\{\pm 2^k,\ \pm\tfrac32\,2^k:k\in\mathbb Z\right\},
\end{align}
which is clearly a superset of  $\Cfp$. Defining 
\begin{align}
	Q_{\infty}(y)=\argmin_{\hat{y}\in\m{C}_{E\infty M1}}|y-\hat{x}|,    
\end{align}
it therefore follows that
\begin{align}
g_\infty(v)&\triangleq \left(1-2^{-v}Q_\infty(2^v)\right)^2=2^{-2v}\left(2^v-Q_\infty(2^v)\right)^2\nonumber\\
&\leq 2^{-2v}\left(2^v-Q_{\FP 4}(2^v)\right)^2=g(v) ,~~\forall v\in\RR.
\label{eq:EIbound}
\end{align}
Consequently
\begin{align}
\phi(2^t)=\int_{t-1}^{t}g(v)\,dv\geq \int_{t-1}^{t}g_\infty(v)\,dv
\stackrel{(a)}{=}\int_1^2 g_\infty(v)\,dv\stackrel{(b)}{=}\phi_{\mathrm{min}},
\label{eq:phiminbound}
\end{align}
where in $(a)$ we have used the fact that  $2^{-{(v+z)}}Q_\infty(2^{(v+z)})=2^{-{v}}Q_\infty(2^{v})$ for all $v\in\RR$ and $z\in\ZZ$ (since $2^z\mathcal \m{C}_{E\infty M1}=\mathcal \m{C}_{E\infty M1}$). We now establish equality $(b)$ above . This follows since for any $y\in[2,4)$ we have that $Q_\infty(y)=\mathrm{round}(y)$, which implies also that for any $1\leq v<2$ it holds that $g_\infty(v)=(1-2^{-v}\mathrm{round}(2^v))^2$. Consequently
\begin{align}
\int_1^2 g_\infty(v)\,dv=\int_1^2 (1-2^{-v}\mathrm{round}(2^v))^2\,dv=\int_0^1 (1-2^{-{(u+1)}}\mathrm{round}(2^{(u+1)}))^2\,du=\phi_{\mathrm{min}}
\end{align}
as required.

The bound~\eqref{eq:phiminbound} holds with equality whenever ~\eqref{eq:EIbound} holds with equality.
Equality in~\eqref{eq:EIbound} in turn holds for all $v$ for which $Q_\infty(2^v)=Q_{\FP 4}(2^v)$. Noting that
\begin{align}
\mathcal \m{C}_{E\infty M1}=\{\ldots, \tfrac34,\ 1,\ \tfrac32,\ 2,\ 3,\ 4,\ 6,\ 8,\ldots \}
\end{align}
we see that 
\begin{align}
Q_\infty(z)=Q_{\FP 4}(z)~~~ \forall z\in[7/8,7].
\end{align}
In particular, for any $y\in\m{I}_{\FP }=[7/4,7]$ we have that $Q_\infty(y)=Q_{\FP 4}(y)$ on the entire interval $[y/2,y]$. Consequently, setting $t_y=\log_2(y)$ we see that for any $y\in\m{I}_{\FP }=[7/4,7]$ it holds that $g(v)=g_\infty(v)$ for all $v\in [t_y-1,t_y]$ and therefore~\eqref{eq:EIbound} holds with equality for all $y\in\m{I}_{\FP }=[7/4,7]$, as claimed.
\end{proof}

\clearpage
\section{Detailed Loss Derivations and SIC Procedure}
\label{app:loss-derivations}

Denote by $\Cscale$ all the positive points in the FP8 E4M3 constellation, and recall the definition of $\Cfp$ from~\eqref{eq:CscaleDef}.

Let $W\in\RR^{m\times n}$ and $X\in\RR^{n\times k}$. We generally think of $X$ as a random vector with $k=1$, but for optimization of the diagonal scaling matrix $A$, we will denote by $X$ the matrix holding $k$ calibration activations in $\RR^n$, so that $k>1$. We also define $\Sigma_X=\frac{1}{k}X X^\top\in\RR^{n\times n}$. Let $L=n/16$ and assume for simplicity that $L$ is an integer. 

We partition $W$ to $L$ matrices of shape $m\times 16$ so that
\begin{align}
W=[W_1|\cdots|W_L],~~~W_\ell=W(:,1+16(\ell-1):16\ell)~\text{for}~\ell=1,\ldots,L.    
\end{align}
Under NVFP4 quantization we represent each $W_\ell$ as
\begin{align}
\hat{W}_\ell=\Gamma_{W_\ell}\cdot\hat{W}_\ell^{\FP4},  
\end{align}
where $\Gamma_{W_\ell}=\diag\left(\gamma_{W_\ell}^{(1)},\ldots,\gamma_{W_\ell}^{(m)} \right)$ with diagonal entries in $\Cscale$ and  $\hat{W}_\ell^{\FP4}\in\Cfp^{m\times 16}$.
Similarly,
\begin{align}
X=\left[\begin{array}{@{}c@{}}
X_1 \\
\hline
\vdots \\
\hline
X_L
\end{array}\right],~~~X_\ell=X(1+16(\ell-1):16\ell,:)~\text{for}~\ell=1,\ldots,L. 
\end{align}
When activations are also quantized, we represent each $X_\ell$ as
\begin{align}
\hat{X}_\ell=\hat{X}_\ell^{\FP4}\cdot \Gamma_{X_\ell},  
\end{align}
where $\Gamma_{X_\ell}=\diag\left(\gamma_{X_\ell}^{(1)},\ldots,\gamma_{X_\ell}^{(k)} \right)$ with diagonal entries in $\Cscale$ and $\hat{X}_\ell^{\FP4}\in\Cfp^{16\times k}$.

In this work we use a diagonal smoothing matrix $A\in\RR^{n\times n}$ so that we actually quantize $W A^{-1}$ and (possibly) $AX$. We keep using the notation $\{\hat{W}_\ell\}$ and $\{\hat{X}_{\ell}\}$ for the corresponding NVFP4 representation even when $A\neq I_n$, in which case $\hat{W}_\ell$ can be very far from $W_\ell$, and same for $X$. We also express the matrix $A$ as a block diagonal matrix with blocks of size $16\times 16$ so that
\begin{align}
A=
\begin{bmatrix}
A_1 & 0 & \cdots & 0 \\
0 & A_2 & \cdots & 0 \\
\vdots & \vdots & \ddots & \vdots \\
0 & 0 & \cdots & A_L
\end{bmatrix}.
\end{align}

\medskip

\textbf{The MatMul Quantization Error:} Let
\begin{align}
e&=\hat{W}\hat{X}-WX=\sum_{\ell=1}^L \left(\hat{W}_\ell\hat{X_\ell}-(W_\ell A_\ell^{-1}) (A_\ell X_\ell)\right)\\
&=\sum_{\ell=1}^L(W_\ell A_\ell^{-1}+e_{W_\ell}) (A_\ell X_\ell+e_{X_\ell})-\left((W_\ell A_\ell^{-1}) (A_\ell X_\ell)\right)\\
&=\sum_{\ell=1}^L e_{W_\ell}A_\ell X_\ell+W_\ell A_\ell^{-1}e_{X_\ell}+e_{W_\ell}e_{X_\ell}
\end{align}
where $e_{W_\ell}=\hat{W}_\ell-(W_\ell A_\ell^{-1})$ and $e_{X_\ell}=\hat{X}_\ell-(A_\ell X_\ell )$. Setting
\begin{align}
e_W=[e_{W_1}|\cdots|e_{W_L}],~~~e_X= \left[\begin{array}{@{}c@{}}
e_{X_1} \\
\hline
\vdots \\
\hline
e_{X_L}
\end{array}\right],    
\end{align}
we obtain
\begin{align}
e=e_W AX+W A^{-1}e_X + e_W e_X.    
\end{align}
We further neglect the third term which has much smaller energy than the sum of the first two under NVFP4 quantization, and approximate
\begin{align}
e\approx e_W AX+W A^{-1}e_X.
\label{derivation:eq:error1storder}
\end{align}

\medskip

\textbf{Randomized FP4 quantization error:} We assume the randomized quantizer $\Qt(\cdot)$ from~\eqref{eq:QtDef} is used. We model it as a channel with input $y\in \RR$ and output $\Qt(y)\in\RR$ satisfying
\begin{align}
\Qt(y)=y+\sqrt{y^2\phi(y)}Z,
\label{derivation:eq:qterrormodel}
\end{align}
where $Z$ is a random variable with zero mean and unit variance, statistically independent of everything. For a matrix $F\in\RR^{k_1\times k_2}$ the notation $\Qt(F)$ corresponds to applying the randomized quantizer $\Qt(\cdot)$ on each component of $F$.

\textbf{Round to nearest (RTN) quantization:} The simplest way to use the randomized quantizer is to feed to it the scaled signal, and multiply the output by the inverse scaling. This is referred to as RTN quantization.
Specifically, given $A,\Gamma_W$ under RTN quantization the entries of the weight matrix $W$ are quantized to
\begin{align}
\hat{W}_\ell^{\FP4}= \Qt\left(\Gamma^{-1}_{W_\ell}W_\ell A^{-1}_\ell\right),~~\text{so that}~~\hat{W}_\ell= \Gamma_{W_\ell}\Qt\left(\Gamma^{-1}_{W_\ell}W_\ell A^{-1}_\ell\right).
\end{align}
Denoting $\{\gamma^{(i)}_{W_{\ell}}\}_{i=1}^{m}$ and $\{\alpha_\ell^{(r)}\}_{r=1}^{16}$ for the diagonal entries of $\Gamma_{W_{\ell}}$ and $A_{\ell}$, respectively, we obtain
\begin{align}
e_{W_{\ell}}(i,r)=\hat{W}_\ell(i,r)-\frac{W_\ell(i,r)}{\alpha^{(r)}_\ell}=\sqrt{\left(\frac{W_\ell(i,r)}{\alpha^{(r)}_\ell}\right)^2\phi\left(\frac{W_\ell(i,r)}{\gamma_{W_\ell}^{(i)}\alpha^{(r)}_\ell} \right)} Z_\ell(i,r),
\end{align}
where $\{Z_\ell(i,r)\}$ are iid zero-mean, unit-variance random variables.

Let
\begin{align}
L^{\mathrm{RTN}}_{W_\ell}(\Gamma_{W_\ell},A_\ell)=\EE\|e_{W_\ell}A_\ell X_\ell\|_F^2,    
\end{align}
where $X_\ell,W_\ell,\Gamma_{W_\ell},A_\ell$ are assumed fixed and the expectation is only with respect to the randomization of $\Qt(\cdot)$, that is, with respect to the iid random variables $\{Z_\ell(i,r)\}$. It immediately follows that
\begin{align}
L^{\mathrm{RTN}}_{W_\ell}(\Gamma_{W_\ell},A_\ell)&=\sum_{i=1}^m\sum_{j=1}^k\sum_{r=1}^{16}W^2_{\ell}(i,r) X^2_{\ell}(r,j)\cdot\phi\left(\frac{W_\ell(i,r)}{\gamma_{W_\ell}^{(i)}\alpha^{(r)}_\ell} \right)\nonumber\\
&=k\sum_{r=1}^{16}\sigma^2_{X_\ell}(r)\sum_{i=1}^m W^2_{\ell}(i,r) \cdot\phi\left(\frac{W_\ell(i,r)}{\gamma_{W_\ell}^{(i)}\alpha^{(r)}_\ell} \right),
\label{derivation:eq:LWRTN}
\end{align}
where $\sigma^2_{X_\ell}(r)=\frac{1}{k}\sum_{j=1}^k X^2_\ell(r,j)$ (we also have $\sigma^2_{X_\ell}(r)=\Sigma_X(r+16(\ell-1),r+16(\ell-1))$).

We assume that the random variables $\{Z_\ell(\cdot,\cdot)\}$ corresponding to different blocks are also independent. This implies
\begin{align}
 L^{\mathrm{RTN}}_W=\EE\|e_W A X\|_F^2=\sum_{\ell=1}^L L^{\mathrm{RTN}}_{W_\ell}(\Gamma_{W_\ell},A_\ell).  
\end{align}

If activations are also quantized, applying RTN with the randomized FP4 quantizer gives
\begin{align}
\hat{X}_\ell^{\FP4}= \Qt\left(A_\ell X_\ell  \Gamma^{-1}_{X_\ell}\right),~~\text{so that}~~\hat{X}_\ell= \Qt\left(A_\ell X_\ell  \Gamma^{-1}_{X_\ell}\right)\Gamma_{X_\ell},
\end{align}
and
\begin{align}
e_{X_{\ell}}(r,j)=\hat{X}_\ell(r,j)-\alpha^{(r)}_\ell X_\ell(r,j)=\sqrt{\left(\alpha^{(r)}_\ell X_\ell(r,j)\right)^2\phi\left(\frac{\alpha^{(r)}_\ell}{\gamma_{X_\ell}^{(j)}} X_\ell(r,j) \right)} \tilde{Z}_\ell(r,j),
\end{align}
where $\{\gamma^{(j)}_{X_{\ell}}\}_{j=1}^{k}$ denote the diagonal entries of $\Gamma_{X_{\ell}}$, and $\{\tilde{Z}_\ell(r,j)\}$ are iid zero-mean unit-variance random variables statistically independent of everything. Similar to the derivation above, we obtain

\begin{align}
L^{\mathrm{RTN}}_{X_\ell}(\Gamma_{X_\ell},A_\ell)=\EE\|W_\ell A^{-1}_\ell e_{X_\ell}\|_F^2=m\sum_{r=1}^{16}\sigma^2_{W_\ell}(r)\sum_{j=1}^k X^2_{\ell}(r,j) \cdot\phi\left(\frac{\alpha^{(r)}_\ell}{\gamma_{X_\ell}^{(j)}} X_\ell(r,j) \right) ,   
\label{derivation:eq:LXRTN}
\end{align}
where $\sigma^2_{W_\ell}(r)=\frac{1}{m}\sum_{i=1}^m W^2_{\ell}(i,r)$. We also have
\begin{align}
 L^{\mathrm{RTN}}_X=\EE\|W A^{-1} e_X\|_F^2=\sum_{\ell=1}^L L^{\mathrm{RTN}}_{X_\ell}(\Gamma_{X_\ell},A_\ell).      
\end{align}

\textbf{Successive interference cancellation (SIC):} Note that
\begin{align}
 \|e_W AX\|_F^2=\tr\left((e_W AX)(e_W AX)^\top\right)=k\tr\left(e_W A\Sigma_X A^\top e_W^\top \right).   
\end{align}
Let $U\in\RR^{n\times n}$ be an upper triangular matrix such that $\Sigma_X=U^\top U$, and denote 
\begin{align}
Y=UA^\top (WA^{-1})^\top=U W^\top\in \RR^{n\times m}.
\end{align}
We have
\begin{align}
 \|e_W AX\|_F^2=k\|U A^\top e_W^\top\|_F^2=k\|Y-UA\hat{W}^\top\|_F^2.   
 \label{derivation:eq:WMSEloss}
\end{align}
Even when $A$ and the E4M3 scaling matrices $\{\Gamma_{W_\ell}\}_{\ell=1}^{L}$ are fixed, the problem of choosing $\hat{W}^{\FP 4}\in \Cfp^{m\times n}$ such as to minimize~\eqref{derivation:eq:WMSEloss} is computationally difficult due to the coupling the matrix $UA$ introduces between different columns of $\hat{W}$ (equivalently, rows of $\hat{W}^\top$). As common in the literature, instead of attempting an optimal solution, we resort to the SIC/GPTQ approach. Taking into account the structure of the NVFP4 constellation, this boils down to choosing $W_{\mathrm{SIC}}\in \Cfp^{n\times m}$ as in Algorithm~\ref{alg:NVFP4SIC}, and then setting $\hat{W}^{\FP 4}=W^\top_{\mathrm{SIC}}$ and $\hat{W}_\ell=\Gamma_{W_\ell} \hat{W}_{\ell}^{\FP 4}$.

\begin{algorithm}[H]
\caption{NVFP4 SIC}
\label{alg:NVFP4SIC}
\begin{algorithmic}
\State \textbf{Inputs:} $Y\in\RR^{n\times m}$, upper triangular $U\in\RR^{n\times n}$,  diagonal matrices $\{A_{\ell}\}_{\ell=1}^{L=n/16}$ with $A_{\ell}=\diag\left(\alpha^{(1)}_{\ell},\ldots,\alpha^{(16)}_{\ell} \right)$,
 and E4M3 diagonal scaling matrices $\{\Gamma_{W_\ell}\}_{\ell=1}^{L=n/16}$ with $\Gamma_{W_\ell}=\diag\left(\gamma^{(1)}_{W_\ell},\ldots,\gamma^{(m)}_{W_\ell} \right)$
\State \textbf{Outputs:} $W_{\mathrm{SIC}}\in \Cfp^{n\times m}$
\vspace{2mm}
\State $Y^{(n)}\gets Y$
\For{$\ell=L:1$}
\For{$r=16:1$}
\State $t\gets 16(\ell-1)+r$
\State $W_{\mathrm{SIC}}(t,:)\gets  Q_{\mathrm{FP4}}\left(\frac{Y^{(t)}(t,:)\Gamma^{-1}_{W_\ell}}{\alpha^{(r)}_\ell  U(t,t)}\right)$\Comment{For optimization use $\Qt$, for actual quantization use $Q_{\mathrm{FP4}}$}
\State $Y^{(t-1)}\gets Y^{(t)}-\alpha_\ell^{(r)}U(:,t)\cdot W_{\mathrm{SIC}}(t,:)\Gamma_{W_\ell}$\Comment{$U(:,t)$ is the $t$th column of $U$}
\EndFor
\EndFor
\end{algorithmic}
\end{algorithm}

The analysis of $L_W^{\mathrm{SIC}}=k\EE\|Y-UA\hat{W}^\top\|_F^2$ under the SIC algorithm is a bit complicated, even under the randomized FP4 quantization error model~\eqref{derivation:eq:qterrormodel}. The reason is that $Z$ entries from time $t$ affect $Y^{(t')}$ for $t'<t$ due to the feedback nature of the quantization process. We therefore make a high-resolution quantization assumption and assume the energy of the (filtered) quantization noise is negligible next to the signal (entries of $Y$) energy. Under this assumption we have that
\begin{align}
Y^{(t)}(t,:)\approx U(t,t)\cdot (W(:,t))^\top.    
\end{align}
With this assumption a straightforward calculation gives (under the randomized FP4 quantization error model~\eqref{derivation:eq:qterrormodel}) for fixed $\{A_\ell\}$ and $\{\Gamma_{W_\ell}\}$
\begin{align}
L_W^{\mathrm{SIC}}=k\sum_{\ell=1}^L \sum_{r=1}^{16}U^2_\ell(r)\sum_{i=1}^m W^2_\ell(i,r)  \phi\left(\frac{W_\ell(i,r)}{\gamma_{W_\ell}^{(i)}\alpha^{(r)}_\ell} \right)=\sum_{\ell=1}^L L_W^{\mathrm{SIC}}(\Gamma_{W_\ell},A_\ell),  
\end{align}
where $U_\ell(r)=U(16(\ell-1)+r,16(\ell-1)+r)$ and
\begin{align}
L_W^{\mathrm{SIC}}(\Gamma_{W_\ell},A_\ell)= k \sum_{r=1}^{16}U^2_\ell(r)\sum_{i=1}^m W^2_\ell(i,r)  \phi\left(\frac{W_\ell(i,r)}{\gamma_{W_\ell}^{(i)}\alpha^{(r)}_\ell} \right).   
\label{derivation:eq:LWSIC}
\end{align}
Note that the only difference between $L_W^{\mathrm{SIC}}(\Gamma_{W_\ell},A_\ell)$ and $L_W^{\mathrm{RTN}}(\Gamma_{W_\ell},A_\ell)$ is replacing $\sigma^2_{X_\ell}(r)$ in the latter with $U^2_\ell(r)$ in the former. Since $U^2_\ell(r)$ is the variance of the optimal prediction error of $X(t)$ - the $t$th coordinate of $X\sim\mathcal{N}(0,\Sigma_X)$, $t=16(\ell-1)+r$ - from $X(t+1:n)$ and $\sigma^2_{X_\ell}(r)$ is the variance of $X(t)$, we have $U^2_\ell(r)\leq \sigma^2_{X_\ell}(r)$ for all $\ell,r$.

\textbf{Total quadratic Distortion:} Using~\eqref{derivation:eq:LWRTN}~\eqref{derivation:eq:LXRTN} and~\eqref{derivation:eq:LWSIC} we obtain the following loss functions for randomized NVFP4 quantization, as a function of $A$ and the E4M3 scales, under the various quantization setups:
\begin{align}
L^{W-\mathrm{RTN}}&=\sum_{\ell=1}^L L^{\mathrm{RTN}}_{W_\ell}(\Gamma_{W_\ell},A_\ell)\label{derivation:eq:PrdictedLossWRTN}\\
L^{W-\mathrm{SIC}}&=\sum_{\ell=1}^L L^{\mathrm{SIC}}_{W_\ell}(\Gamma_{W_\ell},A_\ell)\label{derivation:eq:PrdictedLossWSIC}\\
L^{W-\mathrm{RTN},X-\mathrm{RTN}}&=\sum_{\ell=1}^L \left[L^{\mathrm{RTN}}_{W_\ell}(\Gamma_{W_\ell},A_\ell)+L^{\mathrm{RTN}}_{X_\ell}(\Gamma_{X_\ell},A_\ell)\right]\label{derivation:eq:PrdictedLossWRTNXRTN}\\
L^{W-\mathrm{SIC},X-\mathrm{RTN}}&=\sum_{\ell=1}^L \left[L^{\mathrm{SIC}}_{W_\ell}(\Gamma_{W_\ell},A_\ell)+L^{\mathrm{RTN}}_{X_\ell}(\Gamma_{X_\ell},A_\ell)\right].
\label{derivation:eq:PrdictedLossWSICXRTN}
\end{align}
Note that in all expressions the dependency on $A$ and the E4M3 scales \emph{completely decouples to $L$ different functions, that depend on different parameters}. Therefore, the optimization on $\{A_\ell,\Gamma_{W_\ell},\Gamma_{X_\ell}\}$ decouples to $L$ independent optimization problems.

To obtain the form of~\eqref{eq:lossapproxintro} for, e.g., $L^{W-\mathrm{SIC},X-\mathrm{RTN}}$ set \begin{align}
c_\ell&=k\sum_{r=1}^{16}\sum_{i=1}^m U_\ell^2(r)W_\ell^2(i,r)\\
p_{\ell}(i,r)&=\frac{k U_\ell^2(r)W_\ell^2(i,r)}{c_\ell}\\
\zeta_\ell&=\frac{m\sum_{r=1}^{16}\sum_{j=1}^k \sigma^2_{W_\ell}(r) X^2_{\ell}(r,j)}{c_\ell}\\
q_{\ell}(r,j)&=\frac{m \sigma^2_{W_\ell}(r) X^2_{\ell}(r,j)}{m\sum_{r=1}^{16}\sum_{j=1}^k \sigma^2_{W_\ell}(r) X^2_{\ell}(r,j)}.
\end{align}

\section{Implementation of the Scale Optimization}
\label{app:scale-optimization}

This appendix specifies the continuous optimization of $A$, $\Gamma_W$, and, when activations are quantized, $\Gamma_X$. It also describes how the resulting parameters determine the FP32 tensor scales used for deployment. The continuous optimization and the subsequent selection of E4M3 weight scales are separate steps: the former uses the smoothed loss $\phi$, whereas the latter evaluates deterministic FP4 quantization error.

\paragraph{Calibration statistics and parameterization.}
For a weight matrix $W\in\RR^{m\times n}$, let $k_1$ be the number of calibration activation vectors used to estimate $\Sigma_X$. For W4A4, the activation term uses a fixed subset of $k\leq k_1$ complete activation vectors, sampled uniformly without replacement. This subset is selected once for the mapping and reused in every optimization iteration. Weight-only optimization requires $\Sigma_X$ but does not retain an activation matrix in its objective. When several projections share one smoothing transformation, their weight matrices are stacked along the output dimension and optimized as one mapping; $m$ then denotes the combined number of output rows.

For each group $\ell$, the trainable parameters are the natural logarithms
\begin{align}
 a_{\ell r}&=\log\alpha_\ell^{(r)}, & r&=1,\ldots,16,\nonumber\\
 b_{\ell i}&=\log\gamma_{W_\ell}^{(i)}, & i&=1,\ldots,m,\nonumber\\
 c_{\ell j}&=\log\gamma_{X_\ell}^{(j)}, & j&=1,\ldots,k.
 \label{eq:optimization-log-parameters}
\end{align}
The $c_{\ell j}$ parameters are omitted for W4A16. Scales are reconstructed by exponentiation, ensuring positivity. We initialize $\alpha_\ell^{(r)}=1$, $\gamma_{W_\ell}^{(i)}=\max_r|W_\ell(i,r)|/6$, and $\gamma_{X_\ell}^{(j)}=\max_r|X_\ell(r,j)|/6$, with positive numerical floors for zero-valued blocks. The weights, calibration activations, and second-moment statistics are held fixed throughout this optimization. The loss and trainable log parameters use FP32 arithmetic.

\paragraph{Objective and normalization.}
Define the fixed weight metric
\begin{align}
 d_{\ell r}=
 \begin{cases}
  \Sigma_X(t,t), & \text{for weight RTN},\\
  U(t,t)^2, & \text{for weight SIC},
 \end{cases}
 \qquad t=16(\ell-1)+r.
 \label{eq:optimization-weight-metric}
\end{align}
For robustness, the implementation computes the SIC factor from the damped second moment,
\begin{align}
 U^\top U=\Sigma_X+\lambda I_n,
 \qquad \lambda=0.01\,\frac{\tr(\Sigma_X)}{n},
 \label{eq:optimization-damping}
\end{align}
as is common in GPTQ. 
The per-group losses evaluated in the implementation are
\begin{align}
 \mathcal J_{W_\ell}
 &=\frac{1}{m}\sum_{r=1}^{16}\sum_{i=1}^{m}
 d_{\ell r}W_\ell(i,r)^2
 \phi\!\left(\frac{W_\ell(i,r)}{\alpha_\ell^{(r)}\gamma_{W_\ell}^{(i)}}\right),
 \label{eq:optimization-weight-objective}\\
 \mathcal J_{X_\ell}
 &=\frac{1}{mk}\sum_{r=1}^{16}\sum_{j=1}^{k}
 \left(\sum_{i=1}^{m}W_\ell(i,r)^2\right)X_\ell(r,j)^2
 \phi\!\left(\frac{\alpha_\ell^{(r)}X_\ell(r,j)}{\gamma_{X_\ell}^{(j)}}\right).
 \label{eq:optimization-activation-objective}
\end{align}
Thus, the scalar objective is $\mathcal J=\sum_\ell\mathcal J_{W_\ell}$ for W4A16 and $\mathcal J=\sum_\ell(\mathcal J_{W_\ell}+\mathcal J_{X_\ell})$ for W4A4. These are the corresponding losses in Section~\ref{sec:lossfunc}, normalized by $mk$, with $\{d_{\ell r}\}$ in the weight metric estimated using all $k_1$ calibration vectors and the activation term averaged over the selected $k$ vectors. The coefficients $d_{\ell r}W_\ell(i,r)^2$ and $X_\ell(r,j)^2\sum_iW_\ell(i,r)^2$ are computed once. All groups are evaluated together in a batched computation; their parameters are disjoint, so summing their objectives introduces no coupling between groups.

\paragraph{Differentiable evaluation of $\phi$.}
We evaluate the smoothed relative error through a lookup table indexed by $t=\log_2|x|$. The table has spacing $h=2^{-12}$ on $[-4,8]$, with entries $\phi(2^{t_q})$ at $t_q=-4+qh$. Each entry is computed using the identity proved in Proposition~\ref{prop:phi-log-integral}, by analytically integrating
\begin{align}
 \phi(2^t)=\int_{t-1}^{t}
 \left(1-2^{-v}Q_{\mathrm{FP4}}(2^v)\right)^2\,dv,
 \label{eq:optimization-phi-integral}
\end{align}
splitting the integral at the FP4 decision boundaries. Table construction uses double precision and the cached table used by the optimizer is FP32. Values between adjacent entries are obtained by linear interpolation in $t$. Recall that the trainable variables are the log-scales $a,b,c$ in~\eqref{eq:optimization-log-parameters}. For nonzero inputs, $t$ is affine in these variables, so their gradients follow directly from the interpolation slope with respect to $t$ via the chain rule, with a factor of $1/\ln 2$ converting natural logarithms to base two. Automatic differentiation treats the integer cell index as fixed and differentiates through the interpolation fraction, the logarithm of the magnitude, and the exponential scale parameterization. Consequently, optimization uses the derivative of the interpolated surrogate, without sampling the randomized quantizer or differentiating through deterministic FP4 rounding. The implementation clamps the logarithmic argument to $[-4,8-h]$ before interpolation; outside this interval the lookup contribution has zero derivative with respect to its argument. At $x=0$ the lookup is finite and its loss coefficient is zero.

\paragraph{Adam iterations.}
We jointly update all the log parameters using Adam for $80$ iterations with a fixed learning rate $\eta=0.12$, moment coefficients $(\beta_1,\beta_2)=(0.9,0.999)$, and $\epsilon=10^{-8}$. Weight decay is zero and AMSGrad is disabled. Let $\theta_s$ collect the log parameters $a,b,c$ after iteration $s$, with $c$ omitted for weight-only quantization, and let $g_s=\nabla_\theta\mathcal J(\theta_{s-1})$. Adam maintains exponentially weighted moving averages $p_s$ of the gradients and $v_s$ of their elementwise squares. The updates are
\begin{align}
 p_s&=\beta_1p_{s-1}+(1-\beta_1)g_s, &
 v_s&=\beta_2v_{s-1}+(1-\beta_2)g_s^{\odot2},\nonumber\\
 \widehat p_s&=p_s/(1-\beta_1^s), &
 \widehat v_s&=v_s/(1-\beta_2^s),\nonumber\\
 \theta_{s,j}&=\theta_{s-1,j}-\eta\,
 \frac{\widehat p_{s,j}}{\sqrt{\widehat v_{s,j}}+\epsilon},
 \qquad j=1,\ldots,\dim(\theta),
 \label{eq:optimization-adam}
\end{align}
where $j$ indexes the trainable log parameters, $p_0=v_0=0$, and all operations are elementwise. The quantities $\widehat p_s$ and $\widehat v_s$ correct the moving averages for their zero initialization. Each iteration clears the previous gradients, evaluates the entire fixed objective, backpropagates, and updates the parameters simultaneously. There is no alternating minimization, learning-rate schedule, gradient clipping, or early stopping. The final iterate is retained, rather than the iterate with the smallest observed loss. The procedure seeks a low-loss solution of the nonconvex surrogate and does not impose monotonic descent.

\paragraph{Scale normalization and deployment precision.}
The objective is invariant under a common rescaling within each group. After the last Adam update, we exploit this freedom by setting
\begin{align}
 s_\ell&=\exp\!\left(\frac1m\sum_{i=1}^m
 \log\gamma_{W_\ell}^{(i)}\right),\nonumber\\
 \alpha_\ell^{(r)}&\leftarrow s_\ell\alpha_\ell^{(r)},\qquad
 \gamma_{W_\ell}^{(i)}\leftarrow\gamma_{W_\ell}^{(i)}/s_\ell,\qquad
 \gamma_{X_\ell}^{(j)}\leftarrow s_\ell\gamma_{X_\ell}^{(j)}.
 \label{eq:optimization-canonicalization}
\end{align}
This sets the geometric mean of the continuous weight scales to one while preserving both $W_\ell(i,r)/(\alpha_\ell^{(r)}\gamma_{W_\ell}^{(i)})$ and $\alpha_\ell^{(r)}X_\ell(r,j)/\gamma_{X_\ell}^{(j)}$. The normalization is applied once, after optimization. We then round the diagonal entries of $A$ to the weight dtype when the transformation is folded into a preceding normalization layer, or to BF16 when explicit runtime smoothing is required. The subsequent deterministic E4M3 weight-scale selection uses this rounded $A$. Rounding and discrete scale selection are not part of the Adam iterations.

\paragraph{FP32 tensor scales.}
The global tensor scales are distinct from the block scales optimized above and are not additional Adam parameters. For weights, the saved FP32 global scale is explicitly set to one. After the normalization in~\eqref{eq:optimization-canonicalization}, the weight-scale search selects E4M3 values directly, with the compensating scale already incorporated into $A$.

For W4A4 activations, the optimized $\Gamma_X$ is used only to guide the choice of $A$ and is discarded for deployment. After applying the final rounded diagonal transformation $A_\star$, we recompute the calibration maximum
\begin{align}
 M_X=\max_{1\leq j\leq k_1}\|A_\star x_j\|_\infty.
 \label{eq:optimization-activation-maximum}
\end{align}
This uses the full calibration statistics, not only the subset used in the activation loss. The maximum can be recovered exactly from stored per-channel absolute maxima, since the diagonal scaling is fixed. The saved FP32 global \emph{quantization multiplier} is
\begin{align}
 g_X=\frac{448\cdot6}{M_X},
 \qquad s_X=g_X^{-1}=\frac{M_X}{448\cdot6},
 \label{eq:optimization-activation-global-scale}
\end{align}
where $448$ and $6$ are the maximum positive E4M3 and E2M1 values, respectively, and $s_X$ denotes the corresponding dequantization scale. Numerical safeguards replace nonfinite multipliers by one. This global scale remains fixed during inference. The E4M3 block scales are computed dynamically from each actual group of $16$ smoothed activation entries using absmax-to-6 with this global multiplier. Weight-only configurations do not require activation quantization scales.

\begin{figure}[H]
  \centering
  \begin{subfigure}{0.855\linewidth}
    \centering
    \includegraphics[width=\linewidth]{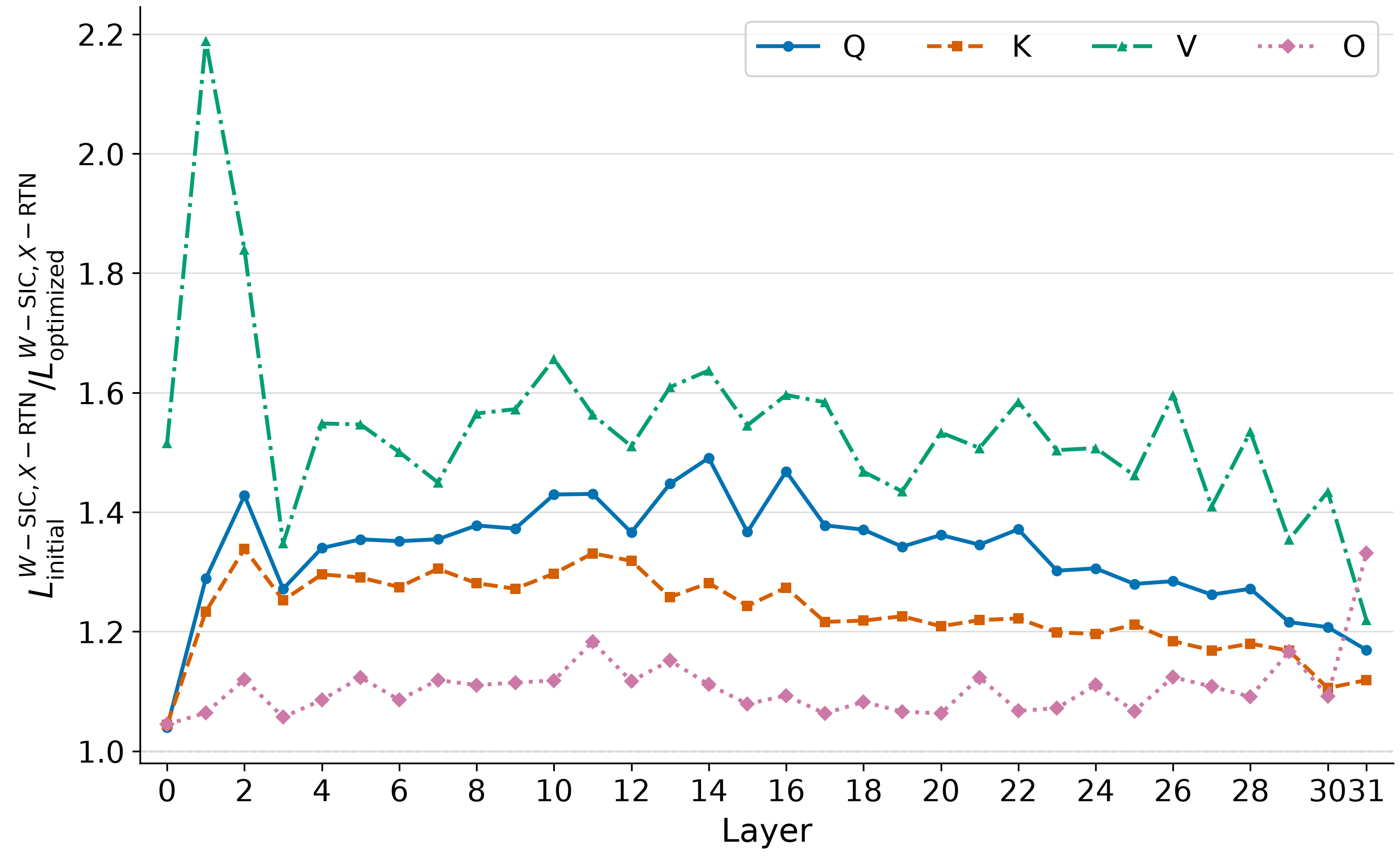}
    \caption{Loss improvement for attention projections.}
    \label{fig:Lhist_attention}
  \end{subfigure}
  \par\medskip
  \begin{subfigure}{0.855\linewidth}
    \centering
    \includegraphics[width=\linewidth]{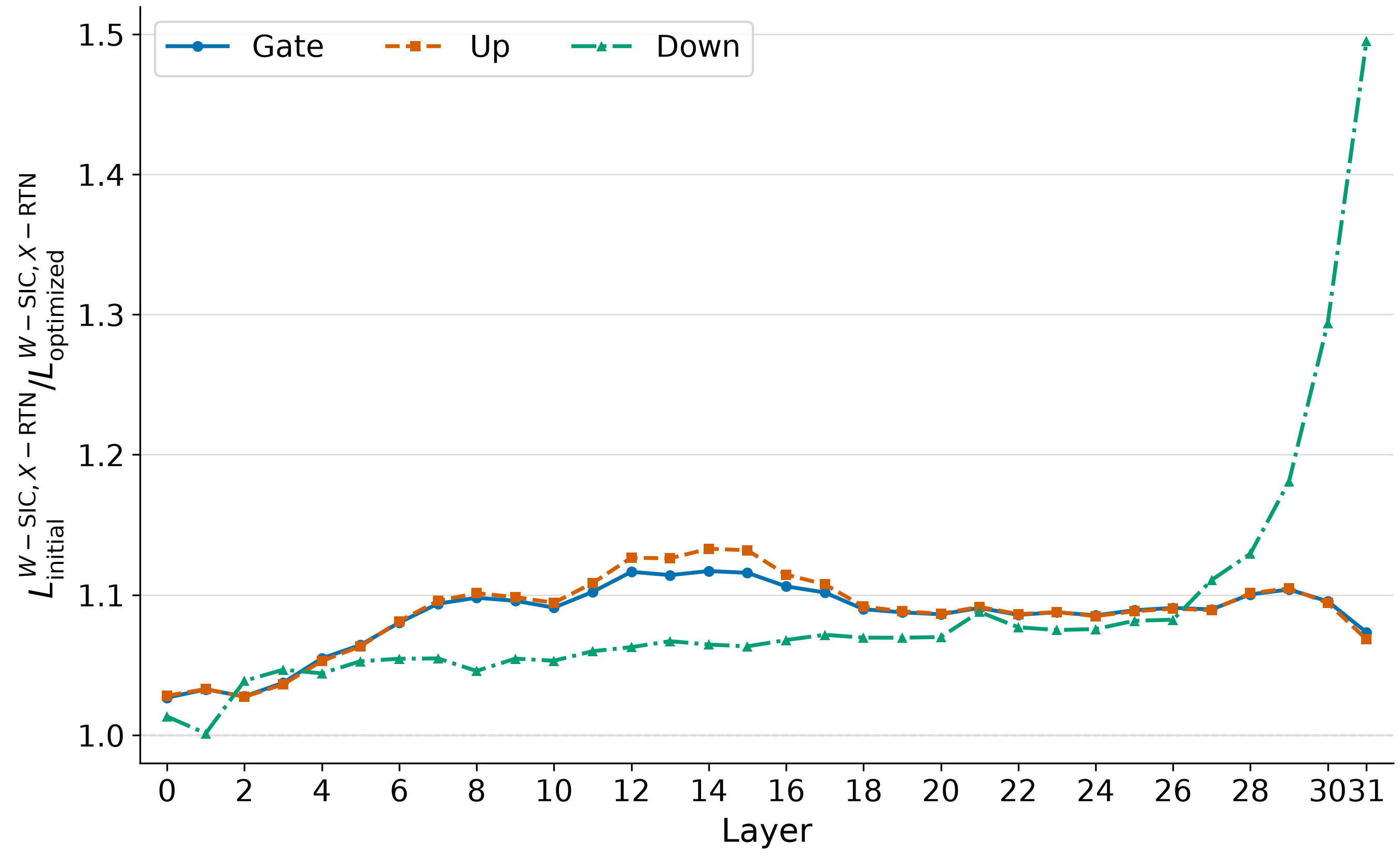}
    \caption{Loss improvement for MLP projections}
    \label{fig:Lhist_mlp}
  \end{subfigure}
  \caption{Ratio between $L^{W-\mathrm{SIC},X-\mathrm{RTN}}$ with default initialized values, and the optimized value, along the 32 layers of Llama-3.1-8B-instruct. The ratio is plotted for each one of the 7 projections.}
  \label{fig:Limprovement}
\end{figure}

\section{Detailed Fine-Grained E4M3 Scale-Selection Procedure}
\label{app:projection-mse}

For each continuous scale $\gamma=\gamma_{W_\ell}^{(i)}$, the implementation searches representable E4M3 candidate scales $\gamma'$ in $[\gamma/1.2,\gamma/0.3]$. The rounded smoothing coefficients $\alpha_\ell^{(r)}$ are held fixed throughout the search.

Specifically, let $\Cscale$ denote the set of positive E4M3 numbers, and for $\gamma\in\RR_+$ let $\Cscale^+(\gamma)=\Cscale\cap[\gamma/1.2,\gamma/0.3]$.

After the continuous optimization and scale normalization in Appendix~\ref{app:scale-optimization}, we choose the final E4M3 weight scales from the candidate sets $\Cscale^+(\gamma_{W_\ell}^{(i)})$. 

Denote by $\Gamma^{\text{final}}_{W_\ell}=\diag(\gamma^{\text{final},(1)}_{W_\ell},....,\gamma^{\text{final},(m)}_{W_\ell})$ the final E4M3 scales used for representing $W_\ell$ in NVFP4 format.
For the $W-\mathrm{RTN}$ and $W-\mathrm{RTN}/X-\mathrm{RTN}$ setups, we set, for each $i=1,\ldots,m$,
\begin{align}
	\gamma^{\text{final},(i)}_{W_\ell}&=\argmin_{\gamma\in \Cscale^+(\gamma_{W_\ell}^{(i)})} \sum_{r=1}^{16}\sigma^2_{X_\ell}(r)\left(W_{\ell}(i,r)- \gamma\alpha_\ell^{(r)} Q_{\FP 4}\left(\frac{W_\ell(i,r)}{\gamma \alpha^{(r)}_\ell} \right)\right)^2,
	\label{eq:gammafinalRTN}
\end{align}
and then set
\begin{align}
	\hat{W}^{\FP 4}_\ell(i,:)=Q_{\FP 4}\left(\frac{W_\ell(i,:) A_\ell^{-1}}{\gamma^{\text{final},(i)}_{W_\ell}}\right)   . 
\end{align}

In the $W-\mathrm{SIC}$ and $W-\mathrm{SIC}/X-\mathrm{RTN}$ setups the dependence of the MSE on $\Gamma^{\text{final}}_{W_\ell}$ is more involved because of the feedback nature of the SIC quantization algorithm. Under RTN, the choice of $\gamma$ affects the quantization errors corresponding to $16$ entries in the corresponding block in an independent manner, and is therefore optimized by~\eqref{eq:gammafinalRTN}. On the other hand, under SIC, errors in quantizing earlier coordinates propagate to the quantization of future coordinates by the SIC feedback loop. Consequently, optimizing $\left\{\Gamma^{\text{final}}_{W_\ell}\right\}_{\ell=1}^L$ is less straightforward. We use Algorithm~\ref{alg:NVFP4SICwithE4M3Optimization} in order to choose $\left\{\Gamma^{\text{final}}_{W_\ell}\right\}_{\ell=1}^L$, and quantize $W$ in the same run. The main challenge in choosing $\gamma$ is that we do not know what will be the input to the quantizer before the feedback quantization procedure was executed. Algorithm~\ref{alg:NVFP4SICwithE4M3Optimization}  chooses the scale $\gamma$ block by block, and relies on the decisions on $\hat{W}$ that were already committed in the previous blocks. When optimizing the scale for block $\ell$, it further \emph{estimates} the effect of intra-block feedback-filtered quantization by approximating $\hat{W}_\ell=W_\ell$ only for the current block.

\begin{algorithm}[!htbp]
	\caption{NVFP4 SIC with Optimized E4M3}
	\label{alg:NVFP4SICwithE4M3Optimization}
	\newcommand{\AlignedAlgorithmComment}[3]{%
  \ifnum#1=1\def\CommentMathWidth{0.40}\def\CommentTextWidth{0.58}%
  \else\def\CommentMathWidth{0.57}\def\CommentTextWidth{0.41}\fi
  \Statex\parbox[t]{\CommentMathWidth\linewidth}{\hspace*{\algorithmicindent}\ifnum#1=2\hspace*{\algorithmicindent}\fi$#2$}%
  \hfill\parbox[t]{\CommentTextWidth\linewidth}{\raggedright\settowidth{\hangindent}{$\triangleright$\ }\hangafter=1\noindent$\triangleright$\ #3\strut\par}}
	\begin{algorithmic}
		\State \textbf{Inputs:} $Y\in\RR^{n\times m}$, upper triangular $U\in\RR^{n\times n}$,  diagonal matrices $\{A_{\ell}\}_{\ell=1}^{L=n/16}$ with $A_{\ell}=\diag\left(\alpha^{(1)}_{\ell},\ldots,\alpha^{(16)}_{\ell} \right)$,
		and initial diagonal scaling matrices $\{\Gamma_{W_\ell}\}_{\ell=1}^{L=n/16}$ with $\Gamma_{W_\ell}=\diag\left(\gamma^{(1)}_{W_\ell},\ldots,\gamma^{(m)}_{W_\ell} \right)$
		\State \textbf{Outputs:} $W_{\mathrm{SIC}}\in \Cfp^{n\times m}$ and scale matrices $\{\Gamma^{\text{final}}_{W_\ell}\}_{\ell=1}^{L=n/16}$
		\vspace{2mm}
		\State $Y^{(n)}\gets Y$
		\For{$\ell=L:1$}
		\AlignedAlgorithmComment{1}{\m{I}_\ell=\{1+16(\ell-1),\ldots,16\ell\}}{The set of $16$ indices participating in $\ell$th block}
		\AlignedAlgorithmComment{1}{U_{\mathrm{block}}=U(\m{I}_\ell,\m{I}_\ell)}{The $\ell$th $16\times 16$ block in the diagonal of $U$}
		\AlignedAlgorithmComment{1}{U^{\mathrm{causal}}_{\mathrm{block}}=U_{\mathrm{block}}-\diag(U_{\mathrm{block}})}{Making $U_{\mathrm{block}}$ strictly upper triangular}
		\State $Y_{\mathrm{block}}\gets Y^{(16\ell)}(\m{I}_\ell,:)$
		\AlignedAlgorithmComment{1}{Y_{\mathrm{block}}^{\mathrm{approx}}\gets Y_{\mathrm{block}}-U^{\mathrm{causal}}_{\mathrm{block}}W^\top_\ell}{Approximating input to the quantizer before applying feedback iteration}
		\For{$i=1:m$}
		\State \begingroup\thinmuskip=0mu\medmuskip=0.5mu\thickmuskip=1mu $\gamma^{\text{final},(i)}_{W_\ell}\gets \argmin\limits_{\gamma\in \Cscale^+(\gamma_{W_\ell}^{(i)})} \sum\limits_{a=1}^{16}\left(Y_{\mathrm{block}}^{\mathrm{approx}}(a,i)- \gamma\alpha_\ell^{(a)} U_{\mathrm{block}}(a,a)\cdot Q_{\FP 4}\left(\frac{Y_{\mathrm{block}}^{\mathrm{approx}}(a,i)}{\gamma \alpha^{(a)}_\ell U_{\mathrm{block}}(a,a)} \right)\right)^2$\endgroup
		\EndFor
		\State $\Gamma^{\text{final}}_{W_\ell}=\diag(\gamma^{\text{final},(1)}_{W_\ell},....,\gamma^{\text{final},(m)}_{W_\ell})$
		\For{$r=16:1$}
		\State $t\gets 16(\ell-1)+r$
		\AlignedAlgorithmComment{2}{W_{\mathrm{SIC}}(t,:)\gets  Q_{\mathrm{FP4}}\left(\frac{Y^{(t)}(t,:)\left(\Gamma^{\text{final}}_{W_\ell}\right)^{-1}}{\alpha^{(r)}_\ell  U(t,t)}\right)}{For optimization use $\Qt$, for actual quantization use $Q_{\mathrm{FP4}}$}
		\AlignedAlgorithmComment{2}{Y^{(t-1)}\gets Y^{(t)}-\alpha_\ell^{(r)}U(:,t)\cdot W_{\mathrm{SIC}}(t,:)\Gamma^{\text{final}}_{W_\ell}}{$U(:,t)$ is the $t$th column of $U$}
		\EndFor
		\EndFor
	\end{algorithmic}
\end{algorithm}

\clearpage
\section{Native NVFP4 Comparison: Experimental Setup}
\label{app:native-benchmark}

Table~\ref{tab:native-nvfp4-comparison} gives the full task results underlying Table~\ref{tab:accuracy-summary}. All 18 rows share one evaluation protocol. We distinguish W4A4 methods using absmax activation scaling, W4A4 methods searching activation scales online, and weight-only W4A16 methods. Bold entries indicate the best score within each group. 

\begin{table}[H]
\caption{Native NVFP4 accuracy on Llama-3.1-8B-Instruct under the common protocol in Appendix~\ref{app:native-benchmark}. All entries are percentages; recovery is relative to the measured BF16 baseline. Bold indicates the best score within each group. FP-Quant rows use author-provided quantization with our native serving adapter; H-Scale W4A4 is our extension. $^\dagger$NVIDIA's separately released checkpoint has uncontrolled preparation. The online-search group includes 4over6, which selects activation scales by reconstruction error during inference. OSFP4 activation-ScaleSweep rows retain their corresponding X-RTN weights, $A$, and tensor scales.}
\centering
\setlength{\tabcolsep}{3pt}
\renewcommand{\arraystretch}{1.08}
\resizebox{\linewidth}{!}{\begin{tabular}{@{}llrrrrrr@{}}
\toprule
\textbf{Precision} & \textbf{Method} & \textbf{MMLU-CoT} & \textbf{GSM8K} & \textbf{HellaSwag} & \textbf{WinoGrande} & \textbf{Avg.} & \textbf{Recovery \%} \\
\midrule
Baseline & BF16 & 72.77 & 85.97 & 80.41 & 77.74 & \textbf{79.22} & \textbf{100.00} \\
\midrule
\multicolumn{8}{@{}l}{\textit{Weights and activations: absmax activation scaling}} \\
W4A4 & RTN (FP-Quant) & 68.61 & 80.29 & 78.87 & 74.82 & 75.65 & 95.48 \\
W4A4 & GPTQ (FP-Quant) & 68.86 & 80.67 & 78.78 & \textbf{77.03} & 76.34 & 96.35 \\
W4A4 & MR-GPTQ (FP-Quant) & \textbf{69.84} & 81.20 & 78.16 & 75.53 & 76.18 & 96.16 \\
W4A4 & SOAR & 69.14 & 82.11 & 78.85 & 76.32 & 76.60 & 96.69 \\
W4A4 & H-Scale (W4A4 extension) & 69.74 & 81.35 & 79.03 & 76.56 & 76.67 & 96.78 \\
W4A4 & NVIDIA released checkpoint$^\dagger$ & 68.49 & 80.89 & 78.75 & 76.32 & 76.11 & 96.07 \\
W4A4 & OSFP4 W-RTN / X-RTN & 69.21 & 82.03 & 78.83 & 75.37 & 76.36 & 96.39 \\
W4A4 & OSFP4 W-SIC / X-RTN & 69.73 & \textbf{82.34} & \textbf{79.09} & 76.95 & \textbf{77.03} & \textbf{97.23} \\
\midrule
\multicolumn{8}{@{}l}{\textit{Weights and activations: online activation-scale search}} \\
W4A4 & ScaleSweepMSE & 68.99 & 82.18 & 78.92 & \textbf{76.40} & 76.62 & 96.72 \\
W4A4 & ScaleSweepMSE + GPTQ & 69.53 & 82.94 & 78.51 & 75.77 & 76.69 & 96.80 \\
W4A4 & 4over6 & 69.56 & 82.11 & 78.56 & 75.14 & 76.34 & 96.36 \\
W4A4 & OSFP4 W-RTN / X-ScaleSweepMSE & \textbf{69.72} & \textbf{83.55} & \textbf{79.36} & 75.22 & 76.96 & 97.14 \\
W4A4 & OSFP4 W-SIC / X-ScaleSweepMSE & 69.56 & 83.40 & 79.10 & 76.32 & \textbf{77.09} & \textbf{97.31} \\
\midrule
\multicolumn{8}{@{}l}{\textit{Weights only}} \\
W4A16 & RTN & 70.69 & 83.85 & 79.36 & 76.80 & 77.67 & 98.04 \\
W4A16 & H-Scale & 71.06 & 83.85 & 79.88 & 77.19 & 78.00 & 98.45 \\
W4A16 & OSFP4 W-RTN & 71.13 & 83.47 & 79.77 & \textbf{78.22} & 78.15 & 98.64 \\
W4A16 & OSFP4 W-SIC & \textbf{71.53} & \textbf{84.15} & \textbf{80.01} & 77.98 & \textbf{78.42} & \textbf{98.98} \\
\bottomrule
\end{tabular}}
\label{tab:native-nvfp4-comparison}
\end{table}

\subsection{Model and calibration}
All locally prepared models start from the same BF16 Llama-3.1-8B-Instruct checkpoint and use the same calibration set: 1024 sequences of 2048 tokens sampled from FineWeb-Edu. All seven linear projections in each decoder layer are quantized. Embeddings, normalization parameters, the language-model head, and KV caches remain in BF16. NVFP4 uses E2M1 values with one E4M3 block scale per 16 input channels and an FP32 tensor scale. W4A4 also quantizes activations in groups of 16, whereas W4A16 retains BF16 activations. No fine-tuning or distillation is performed for the locally prepared models.

\subsection{Methods and implementation provenance}
\paragraph{Author-provided implementations.}
The rows labeled RTN (FP-Quant), GPTQ (FP-Quant), and MR-GPTQ (FP-Quant) use the authors' quantization implementation~\citep{microscaling}, run on our common checkpoint and calibration data. We convert the resulting models for native vLLM evaluation, preserving their quantized weights, scales, and transforms. These rows therefore use author-provided quantization with our serving integration, rather than the original paper's PyTorch evaluation backend. MR-GPTQ includes block Hadamard transforms and activation-based column ordering; GPTQ uses the original column order. All three use absmax activation scaling. SmoothQuant is omitted because the corresponding implementation from the FP-Quant authors was unavailable.

\paragraph{Our implementations of competing methods.}
The weight-only RTN control, ScaleSweepMSE, ScaleSweepMSE+GPTQ, 4over6, SOAR, and H-Scale rows use our own implementations of the respective procedures. ScaleSweepMSE~\citep{lin2026scalesweep} searches weight and activation block scales by reconstruction MSE; ScaleSweepMSE+GPTQ selects weight scales before the GPTQ feedback pass and holds them fixed during feedback. 4over6~\citep{cook2026fouroversix} chooses between absmax-to-4 and absmax-to-6 scales for both weights and activations. SOAR~\citep{bao2026soar} jointly refines tensor and block scales with separate encoding and decoding scales. H-Scale~\citep{yu2026hscale} selects weight scales using activation-weighted reconstruction error. SOAR and the H-Scale W4A4 extension use absmax activation scaling. H-Scale is originally a weight-only method; its W4A4 row is our extension, not a result reported by its authors. These independently implemented competitors use no rotations or column reordering.

\paragraph{OSFP4 and activation-scale ablations.}
The four OSFP4 configurations use the optimization and deterministic scale selection described in Appendices~\ref{app:scale-optimization} and~\ref{app:projection-mse}. The diagonal transformation is folded into preceding normalization layers where possible and applied explicitly before the attention output and MLP down projections otherwise. The two activation-ScaleSweepMSE ablations reuse the corresponding OSFP4 W4A4 checkpoints, changing only the online activation block-scale rule. Their weights, diagonal transformations, and tensor scales remain fixed, so these rows isolate the effect of activation-scale selection.

\paragraph{Released NVIDIA checkpoint.}
We also evaluate the separately released \texttt{nvidia/Llama-3.1-8B-Instruct-NVFP4} checkpoint under the same task protocol. Its calibration and preparation are outside our controlled procedure, so it serves as a released-model reference rather than an isolated comparison of quantization algorithms.

The publishable code package contains the exact asset identities, calibration sampling procedure, software environments, method settings, conversion and validation procedures, and saved results needed to reproduce the table.

\clearpage
\raggedbottom
\section{Additional Models}
\label{app:additional-models}

This appendix reports results on additional models.
All quantized configurations use NVFP4, and all methods
calibrated in our experiments use 1,024 sequences of 2,048
tokens each from FineWeb-Edu.

By default, OSFP4 diagonal smoothing is folded into the
preceding normalization layers for the attention
query/key/value and MLP gate/up projections. For the
attention output and MLP down projections, smoothing is
applied explicitly to their inputs at runtime. Exceptions
for the MoE projections in Qwen3-30B-A3B-Instruct and for
OLMo-3-7B-Instruct are described in their respective
subsections.

Best scores within each precision group are bolded.
All models in this appendix except Qwen3-30B-A3B-Instruct use the same evaluation protocol. For these models, we report prompt-level loose accuracy on IFBench~\cite{pyatkin2025generalizing}. Gemma-4-31B-IT additionally reports strict accuracy, which is excluded from Avg. and Recovery.
Detailed evaluation protocols and generation settings are
provided in the accompanying
\href{https://anonymous.4open.science/r/reproducibility-osfp4-package/README.md}{reproducibility package}.














\subsection{Qwen3-8B}
We additionally evaluate the checkpoint described in the Qwen3-8B-FPQuant-QAT-NVFP4 model card on Hugging Face.
\begin{table}[H]
\setlength{\belowcaptionskip}{0pt}
\caption{Evaluation comparison on Qwen3-8B.}
\label{tab:qwen3-8b-fpquant-osfp4-comparison}

\centering
\small
\setlength{\tabcolsep}{3pt}
\renewcommand{\arraystretch}{1.08}

\resizebox{\linewidth}{!}{\begin{tabular}{@{}ll*{9}{c}@{}}
\toprule
\textbf{Precision} & \textbf{Method} &
\shortstack{\textbf{MMLU-}\\\textbf{CoT}} &
\textbf{GSM8K} &
\shortstack{\textbf{Hella}\\\textbf{Swag}} &
\shortstack{\textbf{Wino}\\\textbf{Grande}} &
\textbf{ARC-C} &
\textbf{GPQA-D} &
\shortstack{\textbf{IFBench}\\\textbf{loose}} &
\textbf{Avg.} &
\shortstack{\textbf{Recovery}\\\textbf{\%}} \\
\midrule

Baseline & BF16 & 71.88 & 90.98 & 76.60 & 70.64 & 91.21 & 48.48 & 30.00 & \textbf{68.54} & \textbf{100.00} \\

\midrule
\multicolumn{11}{@{}l}{\textit{Weights and activations}} \\

W4A4 & RTN (FP-Quant) & 67.68 & 89.31 & 74.68 & 69.46 & 90.36 & 44.95 & 27.33 & 66.25 & 96.66 \\

W4A4 & GPTQ (FP-Quant) & 70.97 & 89.92 & 74.63 & 69.61 & 88.74 & 46.46 & 30.67 & 67.28 & 98.16 \\

W4A4 & MR-GPTQ (FP-Quant) & \textbf{73.59} & 89.76 & 74.92 & 68.43 & 89.93 & 44.95 & 28.00 & 67.08 & 97.87 \\

W4A4 & QAT+Had16 (FP-Quant) & 71.18 & 89.16 & 74.63 & \textbf{70.32} & 90.02 & 43.27 & 28.67 & 66.75 & 97.39 \\

W4A4 & OSFP4 W-RTN / X-RTN & 71.34 & \textbf{90.07} & 74.93 & 69.38 & 90.19 & 47.31 & 28.00 & 67.31 & 98.21 \\

W4A4 & OSFP4 W-SIC / X-RTN & 69.33 & 89.92 & \textbf{74.99} & 69.77 & \textbf{91.21} & \textbf{47.81} & \textbf{31.67} & \textbf{67.81} & \textbf{98.93} \\

\midrule
\multicolumn{11}{@{}l}{\textit{Weights only}} \\

W4A16 & RTN (FP-Quant) & 69.00 & 90.30 & \textbf{75.88} & 69.30 & 90.53 & 45.79 & 29.00 & 67.11 & 97.91 \\

W4A16 & GPTQ (FP-Quant) & 71.81 & 90.14 & 75.46 & 70.48 & 91.04 & 45.29 & 29.67 & 67.70 & 98.77 \\

W4A16 & MR-GPTQ (FP-Quant) & \textbf{74.58} & 90.45 & 75.43 & \textbf{71.19} & \textbf{91.21} & 48.48 & 29.33 & \textbf{68.67} & \textbf{100.19} \\

W4A16 & OSFP4 W-RTN & 72.03 & \textbf{90.67} & 75.76 & 70.48 & 90.02 & 46.63 & \textbf{30.33} & 67.99 & 99.20 \\

W4A16 & OSFP4 W-SIC & 74.13 & 90.07 & 75.79 & 69.93 & 90.61 & \textbf{48.82} & \textbf{30.33} & 68.53 & 99.99 \\

\bottomrule
\end{tabular}}
\end{table}

\subsection{Mistral-7B-Instruct-v0.3}
\label{sec:mistral-accuracy-perplexity}

\begin{table}[H]
\setlength{\belowcaptionskip}{0pt}
\caption{Evaluation comparison on Mistral-7B-Instruct-v0.3.}
\label{tab:mistral-nvfp4}

\centering
\small
\setlength{\tabcolsep}{3pt}
\renewcommand{\arraystretch}{1.08}

\resizebox{\linewidth}{!}{\begin{tabular}{@{}ll*{9}{c}@{}}
\toprule
\textbf{Precision} & \textbf{Method} &
\shortstack{\textbf{MMLU-}\\\textbf{CoT}} &
\textbf{GSM8K} &
\shortstack{\textbf{Hella}\\\textbf{Swag}} &
\shortstack{\textbf{Wino}\\\textbf{Grande}} &
\textbf{ARC-C} &
\textbf{GPQA-D} &
\shortstack{\textbf{IFBench}\\\textbf{loose}} &
\textbf{Avg.} &
\shortstack{\textbf{Recovery}\\\textbf{\%}} \\
\midrule

Baseline & BF16 &
59.29 & 53.98 & 84.80 & 79.56 & 63.82 & 28.11 & 22.00 &
\textbf{55.94} & \textbf{100.00} \\

\midrule
\multicolumn{11}{@{}l}{\textit{Weights and activations}} \\

W4A4 & GPTQ (FP-Quant) &
56.45 & \textbf{50.27} & 83.62 & 79.24 & 62.88 & 26.77 & 18.00 &
53.89 & 96.34 \\

W4A4 & MR-GPTQ (FP-Quant) &
57.44 & 48.22 & 83.63 & \textbf{80.43} & 62.71 & 27.61 & \textbf{21.33} &
54.48 & 97.40 \\

W4A4 & OSFP4 W-RTN / X-RTN &
56.90 & 47.76 & 84.06 & 79.64 & 62.20 & 29.29 & 20.33 &
54.31 & 97.09 \\

W4A4 & OSFP4 W-SIC / X-RTN &
\textbf{57.86} & 49.51 & \textbf{84.16} & 79.16 &
\textbf{63.23} & \textbf{29.80} & 20.00 &
\textbf{54.82} & \textbf{98.00} \\

\midrule
\multicolumn{11}{@{}l}{\textit{Weights only}} \\

W4A16 & GPTQ (FP-Quant) &
56.86 & \textbf{53.30} & 84.30 & 79.79 & 63.65 & \textbf{30.64} & 21.33 &
\textbf{55.70} & \textbf{99.57} \\

W4A16 & MR-GPTQ (FP-Quant) &
58.45 & 52.92 & 83.95 & \textbf{79.95} & 63.65 & 28.96 & \textbf{21.67} &
55.65 & 99.49 \\

W4A16 & OSFP4 W-RTN &
58.21 & 50.57 & \textbf{84.60} & 79.72 & 63.82 & 26.94 & 20.00 &
54.84 & 98.03 \\

W4A16 & OSFP4 W-SIC &
\textbf{58.63} & 50.87 & 84.59 & 79.32 & \textbf{64.51} & 29.46 & 20.33 &
55.39 & 99.02 \\

\bottomrule
\end{tabular}}
\end{table}

\subsection{OLMo-3-7B-Instruct}
\label{sec:olmo-accuracy-perplexity}

OLMo-3-7B-Instruct uses post-normalization. We therefore
apply OSFP4 diagonal smoothing explicitly to projection
inputs at runtime.

\begin{table}[H]
\setlength{\belowcaptionskip}{\baselineskip}
\caption{Evaluation comparison on OLMo-3-7B-Instruct.}
\label{tab:olmo-nvfp4}
\begin{center}
\small
\setlength{\tabcolsep}{3pt}
\renewcommand{\arraystretch}{1.08}

\resizebox{\linewidth}{!}{\begin{tabular}{@{}ll*{9}{c}@{}}
\toprule
\textbf{Precision} & \textbf{Method} &
\shortstack{\textbf{MMLU-}\\\textbf{CoT}} &
\textbf{GSM8K} &
\shortstack{\textbf{Hella}\\\textbf{Swag}} &
\shortstack{\textbf{Wino}\\\textbf{Grande}} &
\textbf{ARC-C} &
\textbf{GPQA-D} &
\shortstack{\textbf{IFBench}\\\textbf{loose}} &
\textbf{Avg.} &
\shortstack{\textbf{Recovery}\\\textbf{\%}} \\
\midrule

Baseline & BF16 &
66.22 & 90.37 & 75.98 & 65.82 & 54.01 & 44.11 & 30.67 &
\textbf{61.03} & \textbf{100.00} \\

\midrule
\multicolumn{11}{@{}l}{\textit{Weights and activations}} \\

W4A4 & GPTQ (FP-Quant) &
65.07 & 90.15 & 75.30 & \textbf{65.46} & \textbf{53.92} & 38.16 & 31.00 &
59.87 & 98.10 \\

W4A4 & MR-GPTQ (FP-Quant) &
64.43 & 89.93 & 75.76 & 64.44 & 53.30 & 38.60 & 31.67 &
59.73 & 97.88 \\

W4A4 & OSFP4 W-RTN / X-RTN &
65.09 & \textbf{90.75} & 75.20 & 65.11 & 53.24 & \textbf{42.26} & 31.67 &
60.47 & 99.10 \\

W4A4 & OSFP4 W-SIC / X-RTN &
\textbf{65.15} & 90.30 & \textbf{75.86} & 65.33 & 53.84 & 39.39 & \textbf{33.67} &
\textbf{60.51} & \textbf{99.15} \\

\midrule
\multicolumn{11}{@{}l}{\textit{Weights only}} \\

W4A16 & GPTQ (FP-Quant) &
65.54 & 89.01 & 75.55 & 65.15 & 53.23 & 41.08 & 32.33 &
60.27 & 98.76 \\

W4A16 & MR-GPTQ (FP-Quant) &
65.67 & 90.30 & 75.04 & 64.57 & \textbf{54.02} & 39.39 & 31.67 &
60.09 & 98.47 \\

W4A16 & OSFP4 W-RTN &
\textbf{66.42} & \textbf{90.83} & 76.09 & 63.46 & 53.58 & 38.55 & \textbf{33.33} &
60.32 & 98.85 \\

W4A16 & OSFP4 W-SIC &
65.90 & 89.61 & \textbf{76.30} & \textbf{66.22} &
53.69 & \textbf{42.26} & 32.00 &
\textbf{60.85} & \textbf{99.72} \\

\bottomrule
\end{tabular}}
\end{center}
\end{table}

\subsection{Gemma-4-31B-IT}
\label{sec:gemma4-31b-it-accuracy}

We additionally evaluate the checkpoint described in
Red Hat AI's Gemma-4-31B-IT-NVFP4 model card on Hugging Face.

\begin{table}[H]
\setlength{\belowcaptionskip}{\baselineskip}
\caption{Evaluation comparison on Gemma-4-31B-IT.}

\label{tab:gemma4-31b-nvfp4}
\begin{center}
\small
\setlength{\tabcolsep}{4pt}
\renewcommand{\arraystretch}{1.08}
\resizebox{\linewidth}{!}{\begin{tabular}{@{}ll*{10}{c}@{}}
\toprule
\textbf{Precision} &
\textbf{Method} &
\shortstack{\textbf{MMLU-}\\\textbf{CoT}} &
\textbf{GSM8K} &
\shortstack{\textbf{Hella}\\\textbf{Swag}} &
\shortstack{\textbf{Wino}\\\textbf{Grande}} &
\textbf{ARC-C} &
\textbf{GPQA-D} &
\shortstack{\textbf{IFBench}\\\textbf{loose}} &
\shortstack{\textbf{IFBench}\\\textbf{strict}} &
\textbf{Avg.} &
\shortstack{\textbf{Recovery}\\\textbf{\%}} \\
\midrule
Baseline & BF16
& 90.45 & 94.39 & 45.66 & 66.38 & 98.04 & 77.27
& 54.33 & 50.00 & \textbf{75.22} & \textbf{100.00} \\

\midrule
\multicolumn{12}{@{}l}{\textit{Weights and activations}} \\

W4A4 & RTN (FP-Quant)
& 89.97 & 94.47 & 43.29 & 66.69 & 97.78 & 75.25
& 53.33 & 49.00 & 74.40 & 98.91 \\

W4A4 & GPTQ (FP-Quant)
& 89.94 & 94.47 & 44.78 & 64.40 & 97.18 & 74.75
& 52.33 & 48.67 & 73.98 & 98.35 \\

W4A4 & MR-GPTQ (FP-Quant)
& 90.04 & 94.16 & 45.89 & 65.90 & 97.87 & 76.26
& 53.67 & 47.67 & 74.83 & 99.48 \\

W4A4 & RTN (RedHatAI)
& 90.04 & \textbf{94.84} & 43.23 & 66.69 & 97.78 & \textbf{78.79}
& 52.67 & 47.67 & 74.86 & 99.53 \\

W4A4 & OSFP4 W-RTN / X-RTN
& 90.09 & 94.39 & \textbf{46.65} & \textbf{66.77} & \textbf{98.04} & 78.28
& 50.67 & 46.33 & 74.99 & 99.69 \\

W4A4 & OSFP4 W-SIC / X-RTN
& \textbf{90.16} & 94.69 & 45.79 & 66.30 & 97.78 & 77.27
& \textbf{54.33} & \textbf{50.00} & \textbf{75.19} & \textbf{99.96} \\

\midrule
\multicolumn{12}{@{}l}{\textit{Weights only}} \\

W4A16 & RTN (FP-Quant)
& 90.27 & 94.31 & 43.31 & 65.04 & 97.53 & \textbf{78.79}
& 51.67 & 47.67 & 74.42 & 98.94 \\

W4A16 & GPTQ (FP-Quant)
& 90.28 & \textbf{94.84} & 45.03 & 65.67 & 97.87 & 78.28
& 52.00 & 47.67 & 74.85 & 99.52 \\

W4A16 & MR-GPTQ (FP-Quant)
& \textbf{90.38} & 94.62 & \textbf{45.66} & \textbf{67.56} & 97.35 & 77.27
& 54.67 & 50.00 & 75.36 & 100.19 \\

W4A16 & OSFP4 W-RTN
& 90.24 & 94.24 & 43.94 & 66.93 & 97.87 & 78.28
& 53.00 & 49.00 & 74.93 & 99.61 \\

W4A16 & OSFP4 W-SIC
& 90.37 & 94.47 & 44.94 & 66.77 & \textbf{98.12} & 77.27
& \textbf{55.67} & \textbf{51.00} & \textbf{75.37} & \textbf{100.21} \\
\bottomrule
\end{tabular}}
\end{center}
\end{table}

\clearpage
\subsection{Qwen3-30B-A3B-Instruct}

For OSFP4, each MoE layer uses expert-specific diagonal smoothing.
During calibration, we preserve the model's natural
MoE routing, so each expert is calibrated only on the tokens
routed to it. Since gate/up smoothing differs
across experts, it cannot be folded into the shared
normalization layer preceding the MoE block. We therefore
apply smoothing explicitly at runtime: after routing, each
selected expert's input is multiplied by its
\(A_{\mathrm{gate/up}}\) before the gate and up projections.
Its intermediate activations are multiplied by
\(A_{\mathrm{down}}\) immediately before the down projection.

For OSFP4 W4A4, we compute separate fixed activation global
scales for each expert's shared input to the gate and up projections
and its down-projection input, following
\eqref{eq:optimization-activation-maximum}--\eqref{eq:optimization-activation-global-scale}.
These scales are computed from each expert's calibration
activations after applying the corresponding diagonal smoothing.

\begin{table}[H]
\setlength{\belowcaptionskip}{\baselineskip}
\caption{Evaluation comparison on Qwen3-30B-A3B-Instruct.}
\label{tab:qwen3-30b-a3b-nvfp4}

\begin{center}
\small
\setlength{\tabcolsep}{4pt}
\renewcommand{\arraystretch}{1.08}

\resizebox{\linewidth}{!}{\begin{tabular}{@{}ll*{9}{c}@{}}
\toprule
\textbf{Precision} &
\textbf{Method} &
\textbf{C-Eval} &
\textbf{LiveBench} &
\textbf{MMLU-R} &
\textbf{AIME 2025} &
\textbf{ARC-C} &
\textbf{BBH} &
\textbf{GPQA-D} &
\textbf{Avg.} &
\textbf{Recovery \%} \\
\midrule
Baseline & BF16 &
86.97 & 70.41 & 88.19 & 63.96 & 95.14 & 90.77 & 55.05 &
\textbf{78.64} & \textbf{100.00} \\
\midrule
\multicolumn{11}{@{}l}{\textit{Weights and activations}} \\
W4A4 & RTN &
\textbf{86.13} & 65.21 & 87.61 & 57.71 & 94.88 & 89.10 &
57.58 & 76.89 & 97.77 \\
W4A4 & RTN + H-Scale &
85.88 & 66.97 & 88.00 & 60.21 & \textbf{95.39} & 89.06 &
56.06 & 77.37 & 98.38 \\
W4A4 & GPTQ + H-Scale &
85.93 & 66.58 & \textbf{88.05} & 59.58 & 95.05 & 89.63 &
57.07 & 77.41 & 98.44 \\
W4A4 & MR-GPTQ + H-Scale &
85.72 & \textbf{68.25} & 87.28 & 57.08 & 94.71 & 89.39 &
55.05 & 76.78 & 97.64 \\
W4A4 & OSFP4 W-RTN / X-RTN &
85.99 & 67.15 & 87.68 & 56.67 & 94.80 & \textbf{89.79} &
\textbf{58.08} & 77.17 & 98.13 \\
W4A4 & OSFP4 W-SIC / X-RTN &
85.99 & 67.38 & 87.47 & \textbf{60.42} & 94.97 & 89.74 &
57.58 & \textbf{77.65} & \textbf{98.74} \\
\midrule
\multicolumn{11}{@{}l}{\textit{Weights only}} \\
W4A16 & RTN &
86.57 & 67.11 & 87.63 & 57.71 & \textbf{95.56} & 89.42 &
\textbf{61.11} & 77.87 & 99.02 \\
W4A16 & RTN + H-Scale &
86.37 & 68.22 & 87.67 & \textbf{61.46} & \textbf{95.56} &
89.54 & 57.07 & 77.98 & 99.16 \\
W4A16 & GPTQ + H-Scale &
86.22 & 67.93 & 87.72 & 59.58 & 95.22 & 90.06 & 55.56 &
77.47 & 98.51 \\
W4A16 & MR-GPTQ + H-Scale &
86.51 & 67.24 & \textbf{87.82} & 60.42 & 95.14 & 89.71 &
54.04 & 77.27 & 98.26 \\
W4A16 & OSFP4 W-RTN &
\textbf{86.70} & \textbf{69.44} & 87.61 & 60.21 & 95.14 &
\textbf{90.08} & 58.08 & \textbf{78.18} & \textbf{99.42} \\
W4A16 & OSFP4 W-SIC &
86.39 & 67.77 & 87.77 & 57.92 & 94.71 & 89.96 & 57.58 &
77.44 & 98.47 \\
\bottomrule
\end{tabular}}
\end{center}
\end{table}

\section{Projection-Level Distortion and Deterministic Scale Selection}
\label{app:projection-mse-experiment}

We compare the predicted randomized loss with measured projection errors for four configurations: $W$-RTN, $W$-SIC, $W$-RTN/$X$-RTN, and $W$-SIC/$X$-RTN. The first two retain unquantized activations (W4A16); the last two quantize both operands (W4A4). We use three fixed projections of Llama-3.1-8B-Instruct: the value projection in layer $12$, the attention output projection in layer $26$, and the MLP up-projection in layer $5$, with zero-based layer indices. Each configuration and projection is optimized independently. The measurements concern local projection error on inputs captured from the BF16 model, rather than error propagated through an end-to-end quantized model.

\paragraph{Empirical activation distribution.}
We use the checkpoint and FineWeb-Edu calibration-token cache specified in Appendix~\ref{app:native-benchmark}. From the $1024$ cached document windows of length $2048$, we select $128$ distinct windows uniformly without replacement and $256$ distinct token positions within each window. Causal BF16 forward passes, without a chat template or generated-token decoding, provide $N=32768$ activation vectors $x_t\in\mathbb R^n$ for each projection. Here $n=4096$ and the output dimensions are $m=1024$, $4096$, and $14336$ for $12v$, $26o$, and $5\mathrm{up}$, respectively. The prompt/token pairs are identical across all configurations and projections. We regard $x$ as uniform on this empirical population and use the same population for optimization and measurement; this is a calibration diagnostic, not a held-out generalization test.

\paragraph{Fitted parameters and comparisons.}
For each projection and configuration, we apply OSFP4 as described in Section~\ref{sec:optimization} to obtain $A,\Gamma_W$ and, for W4A4, $\Gamma_X$. We then keep these fitted parameters fixed and compare one predicted-loss reference with three measured error distributions.
\paragraph{Predicted loss (dashed line).}
For each configuration we plot
\begin{equation}
 \overline L=\frac{L(A,\Gamma_W,\Gamma_X)}{Nm},
 \label{eq:projection-predicted-mse}
\end{equation}
where $L$ is the corresponding loss from~\eqref{eq:PrdictedLossWRTN}--\eqref{eq:PrdictedLossWSICXRTN}, with $\Gamma_X$ omitted for W4A16. This prediction uses the optimized continuous scales, before deterministic E4M3 scale selection, and is the same reference for all curves in a panel.

\paragraph{Deterministic quantization (blue and orange).}
Writing $\widehat W$ for the reconstruction of $WA^{-1}$ and $\widehat x_t$ for that of $Ax_t$, each histogram contains the per-activation output MSE
\begin{equation}
 d_t=\frac1m\|\widehat W\widehat x_t-Wx_t\|_2^2.
 \label{eq:projection-token-mse}
\end{equation}
The two curves test different scale-selection rules:
\begin{itemize}
 \item \textbf{W4A16:} Blue uses OSFP4's deterministic weight-scale selection; orange replaces it with absmax-to-6 weight scales. Both retain the same fitted $A$, tensor scale, and RTN or SIC rounding procedure, with $\widehat x_t=Ax_t$. For SIC, both curves use Algorithm~\ref{alg:NVFP4SICwithE4M3Optimization}, with its scale-selection step replaced by absmax-to-6 for the orange curve.
 
 \item \textbf{W4A4:} Both deterministic curves use the same OSFP4-quantized weights, including the final E4M3 weight scales selected by RTN or SIC. Blue uses ScaleSweepMSE activation block scales, whereas orange uses absmax-to-6. The diagonal transformation $A$ and all tensor scales remain fixed, so the comparison isolates activation block-scale selection.
\end{itemize}

\paragraph{Randomized pseudo-MSE (green).}
Using the fitted continuous scales, we apply the randomized quantizer to weights and, in W4A4, activations, retaining SIC feedback where applicable. To estimate the expectation over the random dither, we repeat this procedure $32$ times with fresh independent dithers, sharing each randomized weight matrix across all activation samples. For each activation, we average the output MSE over these repetitions. The green histogram shows the distribution of these averages across activations.

The full W-RTN/W4A16 comparison is shown in Figure~\ref{fig:projection-mse-rtn-full}; its layer $12$ value projection is shown separately in Figure~\ref{fig:projection-mse-panel} in the main text. The remaining configurations follow. Table~\ref{tab:projection-family-summary} shows that the mean randomized pseudo-MSE closely matches the predicted loss across all configurations and projections, with ratios of $0.96$--$1.01$. Deterministic OSFP4 achieves approximately half the predicted loss, with ratios of $0.48$--$0.53$ (using ScaleSweepMSE activation scales in W4A4).

\begin{figure}[H]
 \centering
 \includegraphics[width=\linewidth]{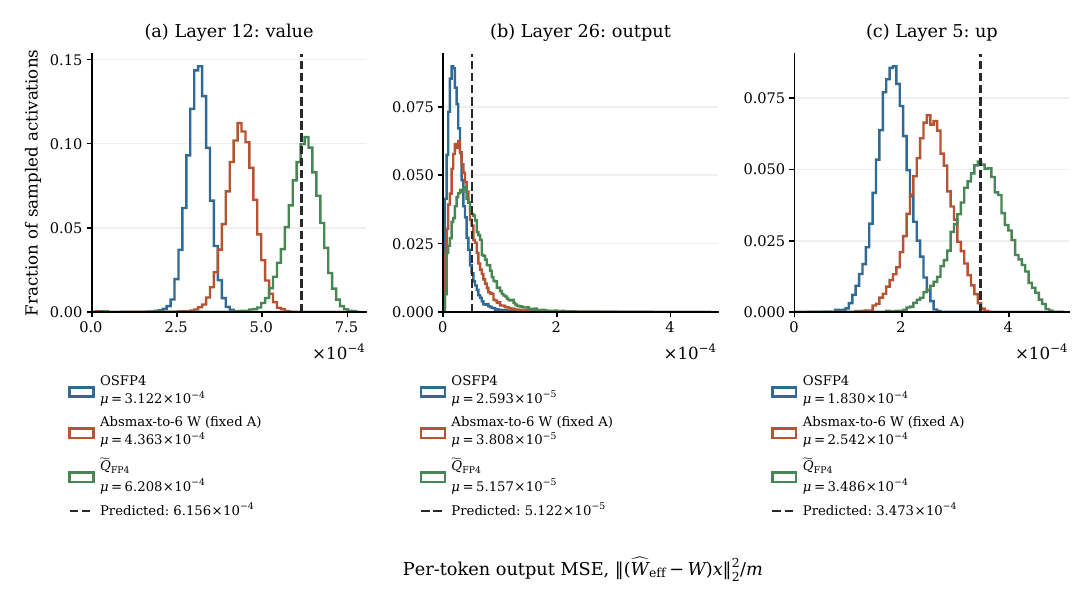}
 \caption{$W$-RTN (W4A16). Blue/orange: OSFP4/absmax-to-6 weight scales at fixed $A$. Green: randomized pseudo-MSE; dashed: predicted loss.}
 \label{fig:projection-mse-rtn-full}
\end{figure}

\begin{figure}[H]
  \centering
  \includegraphics[width=\linewidth]{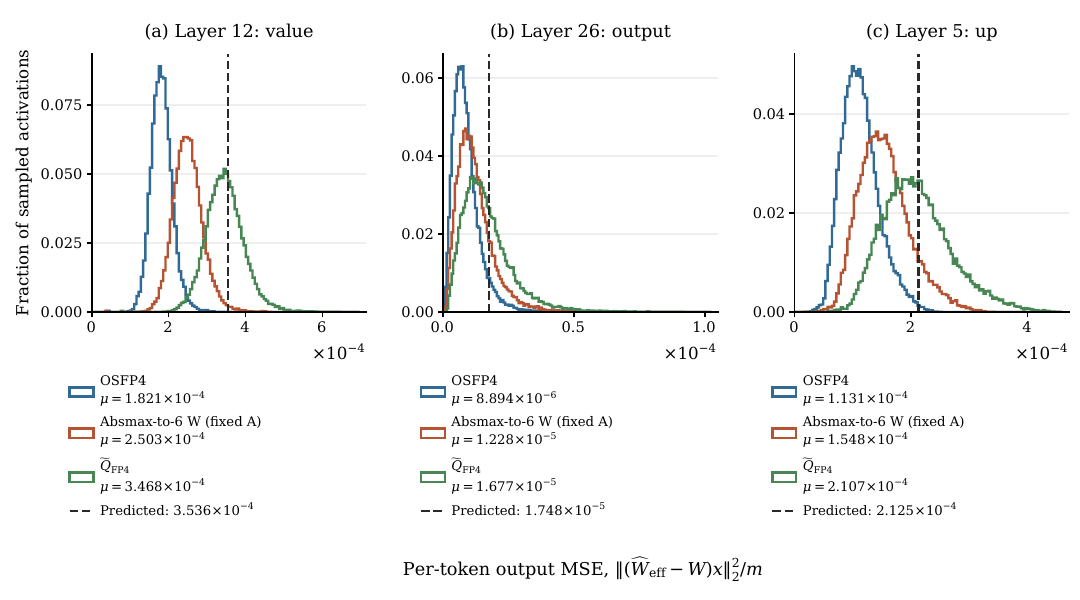}
  \caption{$W$-SIC (W4A16). Blue/orange: OSFP4/absmax-to-6 weight scales at fixed $A$. Green: randomized pseudo-MSE; dashed: predicted loss.}
  \label{fig:projection-mse-sic}
\end{figure}

\begin{figure}[H]
  \centering
  \includegraphics[width=\linewidth]{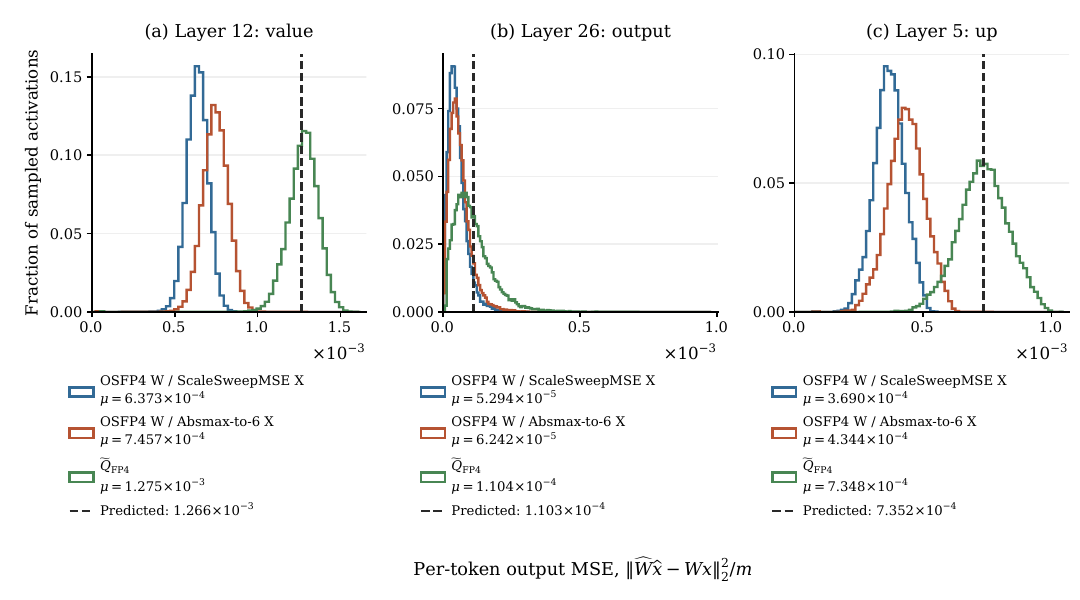}
  \caption{$W$-RTN/$X$-RTN (W4A4). Blue/orange: ScaleSweepMSE/absmax-to-6 activation scales with identical OSFP4 weights. Green: randomized pseudo-MSE; dashed: predicted loss.}
  \label{fig:projection-mse-rtn-a4}
\end{figure}

\begin{figure}[H]
 \centering
 \includegraphics[width=\linewidth]{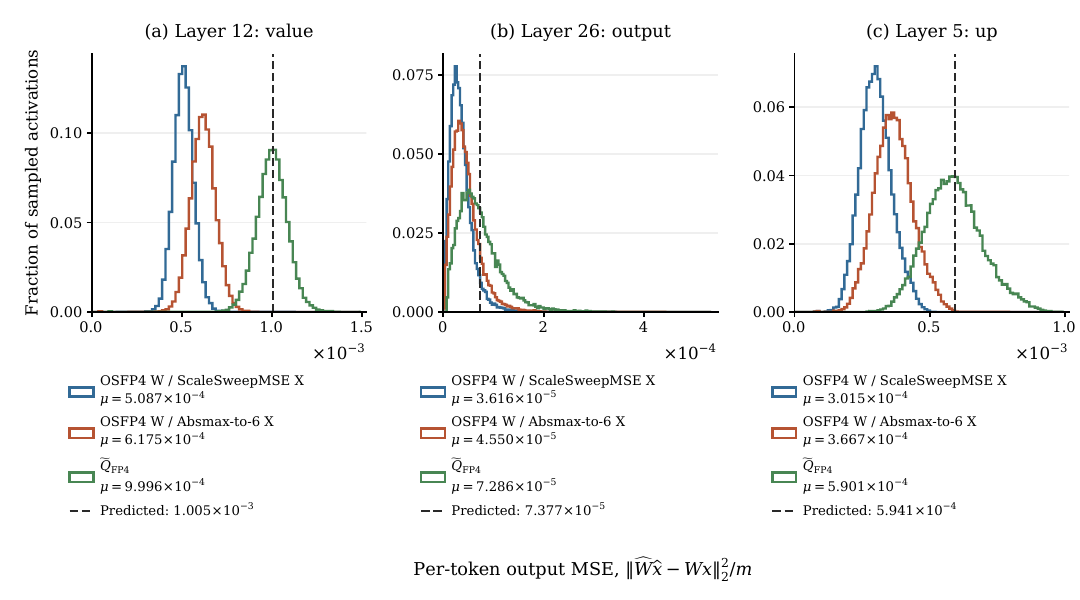}
 \caption{$W$-SIC/$X$-RTN (W4A4). Blue/orange: ScaleSweepMSE/absmax-to-6 activation scales with identical OSFP4 weights. Green: randomized pseudo-MSE; dashed: predicted loss.}
 \label{fig:projection-mse-sic-a4}
\end{figure}

\begin{table}[H]
\caption{Mean measured MSE relative to predicted loss. OSFP4 W4A4 uses ScaleSweepMSE activation scales.}
\centering
\footnotesize
\setlength{\tabcolsep}{3pt}
\renewcommand{\arraystretch}{1.08}
\begin{tabular}{@{}lrrrrrr@{}}
\toprule
& \multicolumn{3}{c}{\textbf{Randomized mean / prediction}} & \multicolumn{3}{c}{\textbf{OSFP4 mean / prediction}} \\
\textbf{Configuration} & 12v & 26o & 5up & 12v & 26o & 5up \\
\midrule
W-RTN & 1.0085 & 1.0068 & 1.0038 & 0.5072 & 0.5062 & 0.5267 \\
W-SIC & 0.9807 & 0.9595 & 0.9913 & 0.5149 & 0.5088 & 0.5320 \\
W-RTN/X-RTN & 1.0075 & 1.0008 & 0.9994 & 0.5035 & 0.4798 & 0.5019 \\
W-SIC/X-RTN & 0.9950 & 0.9876 & 0.9932 & 0.5063 & 0.4902 & 0.5076 \\
\bottomrule
\end{tabular}
\label{tab:projection-family-summary}
\end{table}

\paragraph{Fresh activations.}
The comparisons above fit and measure on the same activations. To test the difference between the effect of the activation scales chosen in inference to the optimized ones, we keep every fitted parameter
fixed - the quantized weights, $A$, $\Gamma_W$, $\Gamma_X$, and the activation
tensor scale - and measure the same errors on $N=32768$ \emph{fresh} activation
vectors. They are captured from 128 of the 896 cached windows not used for
fitting (256 positions each), so they come from the same FineWeb-Edu source but
were never seen by the fit. Blue and orange select
activation block scales from each input, as at deployment.
For the green line, fresh activation $t$ reuses the $\Gamma_X$ fitted for calibration
activation $t$. Figures~\ref{fig:projection-mse-rtn-a4-vs-fresh}
and~\ref{fig:projection-mse-sic-a4-vs-fresh} overlay the fresh results (dashed) on the
calibration results (solid) for the configurations of figures ~\ref{fig:projection-mse-rtn-a4}  and ~\ref{fig:projection-mse-sic-a4}. The solid
curves come from a re-run of the fitting procedure. Their means agree with
figures ~\ref{fig:projection-mse-rtn-a4}  and ~\ref{fig:projection-mse-sic-a4} to within 0.11\%.

For $W$-RTN/$X$-RTN the deterministic errors are essentially unchanged: the ratio
of fresh to calibration mean MSE is 0.998--1.010 for both activation-scale rules
and all three projections. For $W$-SIC/$X$-RTN the errors increase by 6--9\% for
12v and 5up, and by 20--25\% for 26o. RTN rounds each weight independently, so its output
error depends on the inputs only through their per-channel energies, which are
the same for calibration and fresh activations. SIC still outperforms RTN on fresh activations, but
by 11--15\% rather than the 18--32\% measured on the calibration set.

\paragraph{Calibration fitted activation scales.}
With reused $\Gamma_X$, the green mean increases by 1.4--2.1$\times$ for 12v and
5up and by 10--17$\times$ for 26o. Each continuous scale matches the magnitude of
the calibration activation it was fitted for.
Activation scales must therefore be selected online for each input, as
ScaleSweepMSE and absmax-to-6 do.

\begin{figure}[H]
 \centering
 \includegraphics[width=\linewidth]{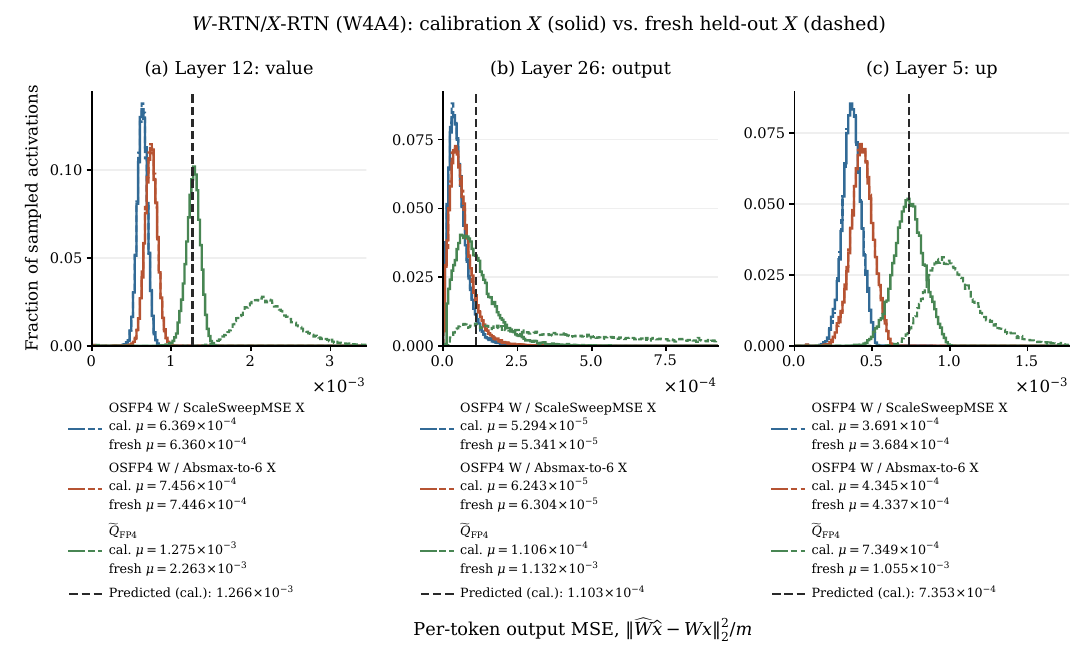}
 \caption{$W$-RTN/$X$-RTN (W4A4) on calibration (solid) and fresh (dashed)
activations. Blue/orange: ScaleSweepMSE/absmax-to-6 activation scales with
identical OSFP4 weights. Green: randomized pseudo-MSE. Fresh activations reuse
the calibration $\Gamma_X$. Black dashed: predicted loss.}
 \label{fig:projection-mse-rtn-a4-vs-fresh}
\end{figure}
\begin{figure}[H]
 \centering
 \includegraphics[width=\linewidth]{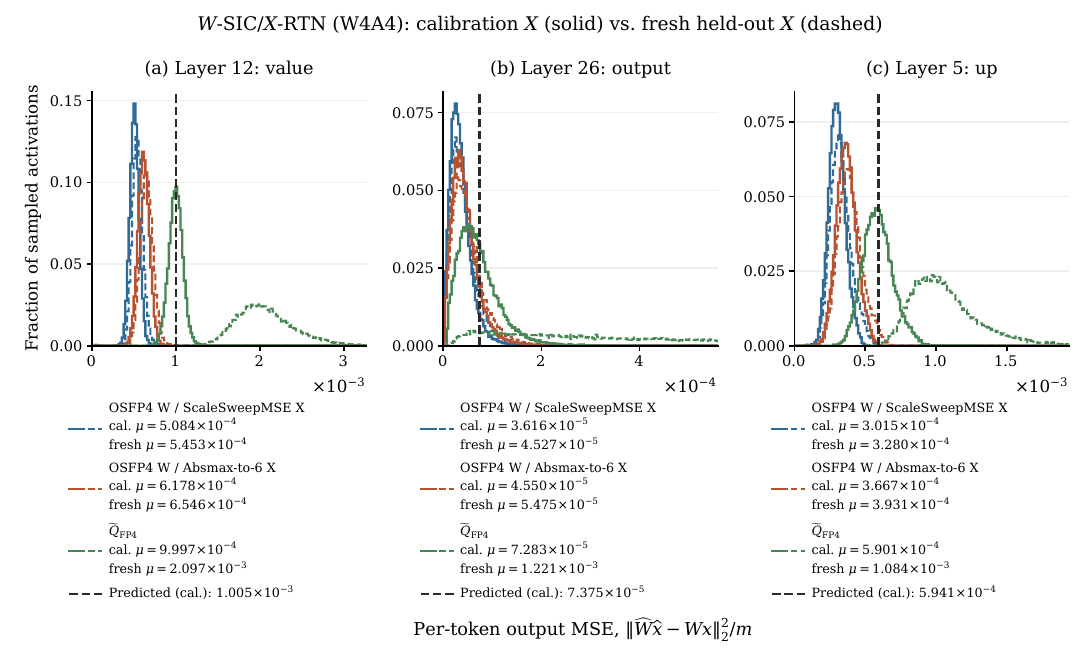}
\caption{$W$-SIC/$X$-RTN (W4A4) on calibration (solid) and fresh (dashed)
activations. Blue/orange: ScaleSweepMSE/absmax-to-6 activation scales with
identical OSFP4 weights. Green: randomized pseudo-MSE. Fresh activations reuse
the calibration $\Gamma_X$. Black dashed: predicted loss.}
 \label{fig:projection-mse-sic-a4-vs-fresh}
\end{figure}
\clearpage

\section{Calibration-set sensitivity}
\label{app:calibration-sets}

GPTQ/SIC rely on calibration data for estimating the second-order statistics of activation vectors that will be encountered in inference runtime. The empirical covariance matrix $\Sigma_X$ of the calibration samples is used for computing the Cholesky factor $U$, which in turn dictates how past quantization errors are fed back before quantizing the next weights in each projection. the W-SIC variants of OSFP4 also need these second order statistics. However, OSFP4 makes further use of the calibration samples beyond second-order statistics: It uses the empirical distribution of activations in order to tune the smoothing matrix $A$. Recall that $A$ aims to place most weights and activations in the interval $\m{I}_{\FP 4}$, and this requires more delicate information on the distribution of $X$ than its covariance matrix (see Figure~\ref{fig:osfp4-scheme}).

We evaluate the sensitivity of the quantization to the calibration set on Llama-3.1-8B-Instruct. We compare FineWeb-Edu, C4, and MathQA under W4A4 and W4A16, keeping the  quantization settings and evaluation protocol fixed. GPTQ and MR-GPTQ use the FP-Quant implementation. 

Each calibration set contains $1024$ sequences of $2048$ tokens. Since MathQA examples are shorter, we shuffle and format them as \texttt{Problem/Options/Rationale/Answer}, tokenize each example, concatenate the resulting token sequences, and divide them into windows of $2048$ tokens. Thus, the token budget is identical for all calibration sets.

Table~\ref{tab:calibration-sets} shows two things: First, OSFP4 dominates the competing GPTQ and MR-GPTQ among all three calibration sets. Second, as expected, it is also more sensitive than GPTQ and MR-GPTQ to the choice of the calibration set. In fact, it's gap from the competing scheme is smallest for FineWeb, which we used for the experiments in this paper.


Prior work shows that calibration-domain choice can produce domain-performance trade-offs, with domain-matched calibration improving the corresponding capabilities ~\citep{he2025preserving}. In our setting, we indeed see that MathQA generally improves mathematical reasoning, with smaller average changes on the other tasks. 
\begin{table}[H]
\caption{Calibration-set sensitivity of quantization on Llama-3.1-8B-Instruct. All entries are percentages.  GPTQ and MR-GPTQ use FP-Quant. Bold indicates the best score among methods for each fixed precision and calibration set.}
\label{tab:calibration-sets}

\centering
\small
\setlength{\tabcolsep}{3pt}
\renewcommand{\arraystretch}{1.08}
\resizebox{0.9\linewidth}{!}{%
\begin{tabular}{@{}lll*{6}{c}@{}}
\toprule
\textbf{Precision} &
\textbf{Method} &
\textbf{Calibration} &
\textbf{MMLU-CoT} &
\textbf{GSM8K} &
\textbf{HellaSwag} &
\textbf{WinoGrande} &
\textbf{Avg.} &
\textbf{Recovery \%} \\
\midrule
Baseline & BF16 & -- &
72.57 & 86.13 & 80.40 & 77.58 &
\textbf{79.17} & \textbf{100.00} \\
\midrule
\multicolumn{9}{@{}l}{\textit{Weights and activations}} \\
W4A4 & OSFP4 W-SIC / X-RTN & FineWeb-Edu &
\textbf{69.95} & 81.96 & \textbf{79.15} & 75.30 &
\textbf{76.59} & \textbf{96.74} \\
W4A4 & OSFP4 W-RTN / X-RTN & FineWeb-Edu &
69.49 & 81.73 & 79.02 & 75.69 &
76.48 & 96.60 \\
W4A4 & GPTQ & FineWeb-Edu &
69.48 & \textbf{82.03} & 78.77 & 75.85 &
76.53 & 96.67 \\
W4A4 & MR-GPTQ & FineWeb-Edu &
69.35 & 81.27 & 78.32 & \textbf{76.95} &
76.47 & 96.59 \\
\addlinespace
W4A4 & OSFP4 W-SIC / X-RTN & C4 &
\textbf{70.10} & 82.56 & 78.78 & 75.61 &
\textbf{76.76} & \textbf{96.96} \\
W4A4 & OSFP4 W-RTN / X-RTN & C4 &
69.41 & 80.82 & \textbf{79.10} & 75.61 &
76.24 & 96.29 \\
W4A4 & GPTQ & C4 &
68.74 & \textbf{82.64} & 78.67 & \textbf{76.56} &
76.65 & 96.82 \\
W4A4 & MR-GPTQ & C4 &
69.07 & 81.35 & 78.58 & 76.24 &
76.31 & 96.39 \\
\addlinespace
W4A4 & OSFP4 W-SIC / X-RTN & MathQA &
\textbf{69.99} & \textbf{84.31} & \textbf{79.19} & \textbf{76.16} &
\textbf{77.41} & \textbf{97.78} \\
W4A4 & OSFP4 W-RTN / X-RTN & MathQA &
69.82 & 83.24 & 79.05 & 75.77 &
76.97 & 97.22 \\
W4A4 & GPTQ & MathQA &
69.40 & 83.24 & 78.55 & 75.69 &
76.72 & 96.91 \\
W4A4 & MR-GPTQ & MathQA &
69.63 & 82.41 & 78.26 & 75.85 &
76.54 & 96.68 \\
\midrule
\multicolumn{9}{@{}l}{\textit{Weights only}} \\
W4A16 & OSFP4 W-SIC & FineWeb-Edu &
\textbf{72.17} & 84.08 & \textbf{79.92} & 77.19 &
\textbf{78.34} & \textbf{98.95} \\
W4A16 & OSFP4 W-RTN & FineWeb-Edu &
71.08 & 83.32 & 79.74 & \textbf{77.98} &
78.03 & 98.56 \\
W4A16 & GPTQ & FineWeb-Edu &
71.38 & \textbf{84.16} & 79.79 & 77.40 &
78.18 & 98.75 \\
W4A16 & MR-GPTQ & FineWeb-Edu &
71.60 & 83.88 & 79.74 & 77.71 &
78.23 & 98.82 \\
\addlinespace
W4A16 & OSFP4 W-SIC & C4 &
\textbf{71.49} & \textbf{84.99} & \textbf{79.95} & 78.14 &
\textbf{78.64} & \textbf{99.33} \\
W4A16 & OSFP4 W-RTN & C4 &
71.04 & 83.32 & 79.72 & \textbf{78.30} &
78.09 & 98.64 \\
W4A16 & GPTQ & C4 &
71.14 & 84.34 & 79.66 & 78.12 &
78.32 & 98.92 \\
W4A16 & MR-GPTQ & C4 &
71.34 & 83.56 & 79.68 & 78.02 &
78.15 & 98.71 \\
\addlinespace
W4A16 & OSFP4 W-SIC & MathQA &
71.45 & \textbf{85.44} & 79.62 & \textbf{77.58} &
\textbf{78.52} & \textbf{99.18} \\
W4A16 & OSFP4 W-RTN & MathQA &
71.28 & 83.17 & \textbf{79.85} & 77.19 &
77.87 & 98.36 \\
W4A16 & GPTQ & MathQA &
71.45 & 85.10 & 79.52 & 77.16 &
78.31 & 98.91 \\
W4A16 & MR-GPTQ & MathQA &
\textbf{71.86} & 84.66 & 79.80 & 77.33 &
78.41 & 99.04 \\
\bottomrule
\end{tabular}%
}
\end{table}

\clearpage
\section{Inference Efficiency}
\label{app:efficiency}

Table~\ref{tab:osfp4-latency-sweep} compares the inference throughput of OSFP4, the vendor NVFP4 RTN baseline, NVFP4 with unit smoothing, and BF16 on Llama-3.1-8B-Instruct (NVIDIA B200, vLLM 0.28.0). Each workload uses a fixed synthetic batch of random-token prompts. Prefill emits one token and measures time to first token. Decode is measured with 256 output tokens and with one output token; subtracting the two durations removes the common prefill contribution. For each method and workload, we perform three independent timing repetitions, each consisting of 30 inference runs. The smoothing transformation $A$ is applied at runtime only on \texttt{o\_proj} and \texttt{down\_proj}.

End-to-end inference includes attention with a BF16 KV cache, activation quantization, output writes, and scheduling overhead in addition to FP4 matrix multiplication. The measured gains therefore depend on the workload. OSFP4 closely matches the unit-smoothing control across the sweep, suggesting that its gap to the vendor NVFP4 path primarily reflects the runtime smoothing implementation rather than the values of $A$.

\begin{table}[H]
\caption{Inference throughput on Llama-3.1-8B-Instruct (NVIDIA B200).}
\centering
\small
\setlength{\tabcolsep}{3pt}
\renewcommand{\arraystretch}{1.08}
\begin{tabular*}{0.9\linewidth}{@{\extracolsep{\fill}}rrrrrrr@{}}
\toprule
\multicolumn{3}{c}{\textbf{Workload}} &
\multicolumn{4}{c}{\textbf{Throughput}} \\
\cmidrule(lr){1-3}\cmidrule(lr){4-7}
\textbf{Batch} & \textbf{Prompt} & \textbf{Out} &
\textbf{BF16} & \textbf{NVFP4} &
\textbf{NVFP4} $A{=}I$ & \textbf{OSFP4} \\
\midrule
\multicolumn{7}{@{}l@{}}{\makebox[0pt][l]{\parbox{0.9\linewidth}{\textit{Prefill: input throughput (thousand tokens/s), measured to first output token}}}} \\
 1 &  512 & 1 & 60.2 &  81.0 &  78.1 &  78.4 \\
 8 &  512 & 1 & 82.5 & 177.0 & 165.4 & 166.5 \\
32 &  512 & 1 & 86.5 & 192.4 & 180.8 & 180.6 \\
 1 & 2048 & 1 & 77.0 & 153.1 & 145.0 & 144.9 \\
 8 & 2048 & 1 & 85.0 & 183.9 & 173.7 & 174.0 \\
 1 & 8192 & 1 & 75.0 & 144.6 & 137.8 & 138.0 \\
\midrule
\multicolumn{7}{@{}l@{}}{\makebox[0pt][l]{\parbox{0.9\linewidth}{\textit{Decode: output throughput across the batch (tokens/s), excluding prefill}}}} \\
 1 & 128 & 256 &  274 &  321 &  312 &  312 \\
 8 & 128 & 256 & 2183 & 2627 & 2548 & 2547 \\
32 & 128 & 256 & 8448 & 9936 & 9664 & 9642 \\
\bottomrule
\end{tabular*}
\label{tab:osfp4-latency-sweep}
\end{table}

\clearpage
\section{Component Ablation}
\label{app:component-ablation}
We conduct an ablation study on Llama-3.1-8B-Instruct under W4A16 and W4A4 quantization. All configurations are calibrated on $1{,}024$ sequences of $2{,}048$ tokens each from FineWeb-Edu. Perplexity and projection-output errors are evaluated on the WikiText-2 test set.

For each method, we report the relative projection MSE for the seven projection types $p\in\{q,k,v,o,\mathrm{gate},\mathrm{up},\mathrm{down}\}$. Let $Y_{t,p}$ denote the full-precision output of projection $p$ in Transformer layer $t$, and let $\hat{Y}_{t,p}^{\mathrm{method}}$ and $\hat{Y}_{t,p}^{\mathrm{FP8~RTN}}$ denote the corresponding outputs produced by the evaluated method and the FP8 RTN reference, respectively. We compute the squared Frobenius-error ratio independently in each Transformer layer and report its mean over all $32$ layers:
\begin{align}
R_p
&=
\frac{1}{32}
\sum_{t=1}^{32}
\frac{
\left\|
\hat{Y}_{t,p}^{\mathrm{method}}-Y_{t,p}
\right\|_F^2
}{
\left\|
\hat{Y}_{t,p}^{\mathrm{FP8~RTN}}-Y_{t,p}
\right\|_F^2
},
~~\text{for}~~
p\in\{q,k,v,o,\mathrm{gate},\mathrm{up},\mathrm{down}\}.
\label{eq:t3-relative-mmse}
\end{align}
The W4A16 configurations use a W8A16 FP8 RTN reference, whereas the W4A4 configurations use a W8A8 FP8 RTN reference.

\begingroup
\setlength{\parskip}{3pt}
All activation quantization uses RTN with AbsMax scaling.

\noindent\textbf{BF16 and RTN.} BF16 denotes the unquantized reference model. The FP8 and NVFP4 RTN baselines use AbsMax scaling for weights.

\noindent\textbf{MSE.} The MSE rows use RTN with \(A=I\) and select each weight block scale by searching all positive finite E4M3 values to minimize unweighted weight reconstruction error.

\noindent\textbf{OSFP4 ablations.} For both \(W\)-RTN and \(W\)-SIC, \(\alpha_{\mathrm{only}}\) computes the block scales using absmax-to-6 and optimizes our loss function only with respect to \(A\). The \(\phi_{\mathrm{exact}}\) ablation replaces \(\phi\) in the optimization objective with its deterministic counterpart shown in Figure~\ref{fig:phi-fp4}. All other steps match the corresponding full method.

\noindent\textbf{SIC.} Algorithm~\ref{alg:NVFP4SIC} with \(A=I\) and AbsMax weight scales as inputs. No diagonal smoothing or scale optimization is performed.

\noindent\textbf{Full OSFP4 methods.} The \(W\)-RTN and \(W\)-SIC rows correspond to our four OSFP4 methods.
\par
\endgroup

\begin{table}[H]
\centering
\captionsetup{skip=4pt}\setlength{\belowcaptionskip}{0pt}
\captionsetup{font=normalsize,skip=4pt}
\caption{Component ablation on Llama-3.1-8B-Instruct:
relative MSE and WikiText-2 perplexity.}
\label{tab:llama31-8b-nvfp4-ablation}
\setlength{\belowcaptionskip}{0pt}

\normalsize
\setlength{\tabcolsep}{2pt}
\renewcommand{\arraystretch}{1}

\scalebox{0.85}{%
\begin{tabular*}{\linewidth}{@{\extracolsep{\fill}}ll*{8}{c}@{}}
\toprule
\textbf{Precision} & \textbf{Method} &
\multicolumn{7}{c}{\textbf{Relative MSE} \(\downarrow\)} &
\textbf{PPL} \(\downarrow\) \\
\cmidrule(lr){3-9}
& &
\(\boldsymbol{q}\) &
\(\boldsymbol{k}\) &
\(\boldsymbol{v}\) &
\(\boldsymbol{o}\) &
\textbf{gate} &
\textbf{up} &
\textbf{down} & \\
\midrule

Baseline & BF16
& 0.00 & 0.00 & 0.00 & 0.00 & 0.00 & 0.00 & 0.00 & 7.22 \\

\specialrule{1pt}{0.2em}{0.2em}
\multicolumn{10}{@{}l}{\textit{Weights and activations}} \\
\midrule

W8A8 & RTN
& 1.00 & 1.00 & 1.00 & 1.00 & 1.00 & 1.00 & 1.00 & 7.27 \\

W4A4 & RTN
& 10.86 & 8.97 & 16.18 & 12.67 & 12.41 & 12.85 & 12.45 & 7.88 \\

W4A4 & MSE
& 9.76 & 8.09 & 15.00 & 10.96 & 10.80 & 11.20 & 10.81 & 7.85 \\

W4A4 & \({W\text{-RTN / }X\text{-RTN}}_{\textstyle\phi_{\text{\scriptsize exact}}}\)
& 7.34 & 6.22 & 11.12 & 10.34 & 10.45 & 10.95 & 10.01 & 7.75 \\

W4A4 & \({W\text{-RTN / }X\text{-RTN}}_{\textstyle\alpha_{\text{\scriptsize only}}}\)
& 7.23 & 6.18 & 10.89 & 10.18 & 10.43 & 10.84 & 9.90 & 7.75 \\

W4A4 & \({W\text{-RTN / }X\text{-RTN}}\)
& 7.05 & 6.19 & 10.12 & 10.12 & 10.38 & 10.75 & 9.77 & 7.77 \\

W4A4 & SIC
& 8.52 & 6.88 & 13.54 & 9.25 & 10.75 & 11.28 & 10.68 & 7.77 \\

W4A4 & \({W\text{-SIC / }X\text{-RTN}}_{\textstyle\phi_{\text{\scriptsize exact}}}\)
& 6.33 & \textbf{5.32} & 9.71 & 8.15 &
9.37 & 9.86 & 8.97 & 7.68 \\

W4A4 & \({W\text{-SIC / }X\text{-RTN}}_{\textstyle\alpha_{\text{\scriptsize only}}}\)
& 6.29 & \textbf{5.32} & 9.49 & 8.07 &
9.40 & 9.81 & 8.94 & 7.69 \\

W4A4 & \({W\text{-SIC / }X\text{-RTN}}\)
& \textbf{6.17} & 5.38 & \textbf{8.94} &
\textbf{8.03} & \textbf{9.33} & \textbf{9.72} &
\textbf{8.82} & \textbf{7.67} \\

\specialrule{1pt}{0.2em}{0.2em}
\multicolumn{10}{@{}l}{\textit{Weights only}} \\
\midrule

W8A16 & RTN
& 1.00 & 1.00 & 1.00 & 1.00 & 1.00 & 1.00 & 1.00 & 7.24 \\

W4A16 & RTN
& 10.27 & 8.68 & 14.30 & 12.59 & 12.22 & 12.64 & 12.52 & 7.55 \\

W4A16 & MSE
& 8.12 & 6.93 & 12.01 & 9.28 & 9.08 & 9.39 & 9.43 & 7.52 \\

W4A16 & \({W\text{-RTN}}_{\textstyle\phi_{\text{\scriptsize exact}}}\)
& 5.46 & 4.51 & 9.03 & 8.25 & 8.51 & 9.01 & 8.24 & 7.44 \\

W4A16 & \({W\text{-RTN}}_{\textstyle\alpha_{\text{\scriptsize only}}}\)
& 5.42 & 4.43 & 9.28 & 8.21 & 8.50 & 9.00 & 8.22 & 7.44 \\

W4A16 & \({W\text{-RTN}}\)
& 5.29 & 4.42 & 8.53 & 8.16 & 8.48 & 8.95 & 8.10 & 7.42 \\

W4A16 & SIC
& 5.88 & 4.73 & 9.05 & 5.82 & 9.00 & 9.53 & 9.19 & 7.45 \\

W4A16 & \({W\text{-SIC}}_{\textstyle\phi_{\text{\scriptsize exact}}}\)
& 3.67 & 2.90 & 6.21 & \textbf{4.02} &
6.46 & 6.87 & 6.33 & 7.40 \\

W4A16 & \({W\text{-SIC}}_{\textstyle\alpha_{\text{\scriptsize only}}}\)
& 3.67 & \textbf{2.88} & 6.26 & \textbf{4.02} &
6.46 & 6.87 & 6.31 & \textbf{7.36} \\

W4A16 & \({W\text{-SIC}}\)
& \textbf{3.63} & 2.89 & \textbf{6.08} &
\textbf{4.02} & \textbf{6.45} & \textbf{6.86} &
\textbf{6.29} & 7.39 \\

\bottomrule

\end{tabular*}
}
\end{table}

We generally see that the full OSFP4 method obtains the best performance. It is significantly better than the baseline schemes. However, both \(\phi_{\mathrm{exact}}\) and  \(\alpha_{\mathrm{only}}\) OSFP4 ablations the  are quite close. The ablation hints that the main advantage of using $\phi(\cdot)$ over $\phi_{\mathrm{exact}}$ is conceptual: using $\phi$ is consistent with the model~\eqref{eq:qterrormodel} that approximates the multiplicative-dither quantizer, whereas using~\eqref{eq:qterrormodel} with $\phi_{\mathrm{exact}}$ instead of $\phi$ for modeling the deterministic quantizer is far from justifiable. In addition, using $\phi$ makes it very clear that the objective in choosing $A,\Gamma_W,\Gamma_X$ is to place the majority of entries in $\m{I}_{\FP 4}$.  The exact objective $\phi_{\mathrm{exact}}$ on the other hand, does not reveal this feature. Nevertheless, using it in the optimization will result in very mild performance degradation.

The \(\alpha_{\mathrm{only}}\) ablation shows that optimizing the diagonal smoothing matrix \(A\) accounts for most of the performance improvement, while joint continuous optimization of the block scales provides a smaller additional MSE gain. This behavior is consistent with the structure of the FP4 quantization objective: diagonal smoothing can redistribute the relative magnitudes of entries within each quantization block, whereas adjusting a block scale changes all normalized entries by a common factor. Since absmax-to-6 already places the largest entry of each block within the minimum-distortion interval of \(\phi\), optimizing \(A\) addresses much of the remaining imbalance. Importantly, the \(\alpha_{\mathrm{only}}\) ablation retains our final deterministic E4M3 scale-selection procedure. Thus, the comparison isolates the contribution of joint continuous scale optimization, rather than the contribution of block-scale optimization altogether. The small additional improvements from the full method suggest that the deterministic scale-selection stage is effective even when initialized from absmax scales. Full optimization nevertheless produces consistently lower mean relative projection MSE across the four configurations.

\paragraph{Loss versus measured MSE.}
Table~\ref{tab:llama31-8b-loss-ablation} compares the value of the optimization objective with the measured projection
error for the $\alpha_{\text{only}}$ ablation (subscript $\alpha$) and the full OSFP4 method with full $\alpha$ and $\Gamma$ optimization (subscript $\mathrm{full}$). $L$ is the loss of the corresponding configuration: \ref{eq:PrdictedLossWRTN} for
$W$-RTN, \ref{eq:PrdictedLossWSIC} for $W$-SIC, \ref{eq:PrdictedLossWRTNXRTN} for $W$-RTN/$X$-RTN and \ref{eq:PrdictedLossWSICXRTN} for $W$-SIC/$X$-RTN. It is
evaluated on the FineWeb-Edu calibration data at the fitted $A$ and continuous scales,
before deterministic E4M3 scale selection, and normalized per output element as in \ref{eq:projection-predicted-mse}.
$E$ is the measured projection-output MSE of the deployed model on WikiText-2, as in \ref{eq:t3-relative-mmse}. The results show that the gain for allowing to also optimize $\Gamma_W$ and $\Gamma_X$ in the loss approximation $\boldsymbol{L}_{\mathrm{full}}$ vs. $\boldsymbol{L}_{\alpha}$ is significantly larger than in the final measured projection-output MSE ($\boldsymbol{E}_{\mathrm{full}}$ vs. $\boldsymbol{E}_{\mathrm{\alpha}}$). This is explained by the fact that for the final quantization the scales $\Gamma_W$ and $\Gamma_X$ undergo the same final optimization under both setups. 

\begin{table}[H]
\centering
\captionsetup{skip=4pt}\setlength{\belowcaptionskip}{0pt}
\caption{Loss and measured MSE by projection and quantization method, normalized by the corresponding FP8 RTN measured MSE for each projection.}
\label{tab:llama31-8b-loss-ablation}
\small
\setlength{\tabcolsep}{2pt}
\renewcommand{\arraystretch}{1.15}
\scalebox{0.93}{%
\begin{tabular*}{\linewidth}{@{\extracolsep{\fill}}l*{9}{c}@{}}
\toprule
\multicolumn{10}{@{}l}{\textit{Weights and activations}} \\
\midrule
\textbf{Projection} & \textbf{W8A8 RTN}
& \multicolumn{4}{c}{\textbf{W4A4 W\text{-RTN / }X\text{-RTN}}}
& \multicolumn{4}{c}{\textbf{W4A4 W\text{-SIC / }X\text{-RTN}}} \\
\cmidrule(lr){2-2} \cmidrule(lr){3-6} \cmidrule(lr){7-10}
& \(\boldsymbol{E}\)
& \(\boldsymbol{L}_{\alpha}\) & \(\boldsymbol{L}_{\mathrm{full}}\) & \(\boldsymbol{E}_{\alpha}\) & \(\boldsymbol{E}_{\mathrm{full}}\)
& \(\boldsymbol{L}_{\alpha}\) & \(\boldsymbol{L}_{\mathrm{full}}\) & \(\boldsymbol{E}_{\alpha}\) & \(\boldsymbol{E}_{\mathrm{full}}\) \\
\midrule
\(\boldsymbol{q}\) & 1.00 & 19.53 & 17.24 & 7.27 & 7.08 & 14.40 & 12.67 & 6.32 & \textbf{6.19} \\
\(\boldsymbol{k}\) & 1.00 & 19.22 & 17.68 & 6.22 & 6.22 & 13.68 & 12.49 & \textbf{5.35} & 5.40 \\
\(\boldsymbol{v}\) & 1.00 & 23.00 & 18.43 & 10.91 & 10.15 & 17.79 & 14.75 & 9.51 & \textbf{8.96} \\
\(\boldsymbol{o}\) & 1.00 & 18.24 & 17.73 & 10.18 & 10.12 & 12.63 & 12.41 & 8.07 & \textbf{8.03} \\
\textbf{gate} & 1.00 & 18.31 & 17.30 & 10.39 & 10.34 & 14.83 & 14.12 & 9.36 & \textbf{9.29} \\
\textbf{up} & 1.00 & 18.49 & 17.44 & 10.81 & 10.73 & 15.19 & 14.47 & 9.79 & \textbf{9.69} \\
\textbf{down} & 1.00 & 25.04 & 24.35 & 9.89 & 9.76 & 19.88 & 19.52 & 8.94 & \textbf{8.81} \\

\specialrule{1pt}{0.2em}{0.2em}
\multicolumn{10}{@{}l}{\textit{Weights only}} \\
\midrule
\textbf{Projection} & \textbf{W8A16 RTN}
& \multicolumn{4}{c}{\textbf{W4A16 W\text{-RTN  }}}
& \multicolumn{4}{c}{\textbf{W4A16 W\text{-SIC  }}} \\
\cmidrule(lr){2-2} \cmidrule(lr){3-6} \cmidrule(lr){7-10}
& \(\boldsymbol{E}\)
& \(\boldsymbol{L}_{\alpha}\) & \(\boldsymbol{L}_{\mathrm{full}}\) & \(\boldsymbol{E}_{\alpha}\) & \(\boldsymbol{E}_{\mathrm{full}}\)
& \(\boldsymbol{L}_{\alpha}\) & \(\boldsymbol{L}_{\mathrm{full}}\) & \(\boldsymbol{E}_{\alpha}\) & \(\boldsymbol{E}_{\mathrm{full}}\) \\
\midrule
\(\boldsymbol{q}\) & 1.00 & 16.73 & 16.12 & 5.47 & 5.33 & 7.53 & 7.40 & 3.69 & \textbf{3.66} \\
\(\boldsymbol{k}\) & 1.00 & 16.85 & 16.59 & 4.46 & 4.46 & 6.72 & 6.68 & \textbf{2.90} & \textbf{2.90} \\
\(\boldsymbol{v}\) & 1.00 & 20.36 & 18.13 & 9.32 & 8.56 & 11.06 & 10.64 & 6.28 & \textbf{6.09} \\
\(\boldsymbol{o}\) & 1.00 & 16.28 & 16.09 & 8.21 & 8.16 & 6.38 & 6.39 & 4.03 & \textbf{4.02} \\
\textbf{gate} & 1.00 & 16.34 & 16.25 & 8.43 & 8.42 & 10.17 & 10.10 & 6.41 & \textbf{6.40} \\
\textbf{up} & 1.00 & 16.66 & 16.50 & 8.96 & 8.91 & 10.80 & 10.71 & 6.84 & \textbf{6.83} \\
\textbf{down} & 1.00 & 18.59 & 18.16 & 8.22 & 8.10 & 11.00 & 10.82 & 6.31 & \textbf{6.28} \\
\bottomrule
\end{tabular*}
}
\end{table}

\end{document}